\documentclass{article}

\usepackage{microtype}
\usepackage{xspace}
\usepackage{graphicx}
\usepackage{subfigure}
\usepackage{booktabs} 
\usepackage{enumitem}
\usepackage{xspace}

\usepackage[hypertexnames=false ]{hyperref}

\usepackage[accepted]{icml2023}

\usepackage{amsmath}
\usepackage{amssymb}
\usepackage{mathtools}
\usepackage{amsthm}

\usepackage[textsize=tiny]{todonotes}

\newtheorem{example}{Example}

\newcommand{\CG}[1]{\textcolor{green}{CG:#1}}

\renewcommand{\hat}{\widehat}
\newcommand{\cX}{\ensuremath{\mathcal{X}}}
\newcommand{\cA}{\ensuremath{\mathcal{A}}}
\newcommand{\atrue}{\ensuremath{\bar{a}}}
\newcommand{\E}{\ensuremath{\mathbb{E}}}

\newcommand{\super}{\ensuremath{\xi}}
\renewcommand{\P}{\ensuremath{\mathbb{P}}}
\DeclareMathOperator*{\argmax}{\ensuremath{\text{argmax}}}
\newcommand{\one}{\mathbbm{1}}

\newcommand{\Rhat}{\ensuremath{\widehat{R}}}
\newcommand{\lag}{\ensuremath{\mathcal{L}}}
\newcommand{\lagh}{\widehat\lag}

\newcommand{\reg}{\ensuremath{\text{Reg}}}
\newcommand{\regh}{\ensuremath{\widehat{\reg}}}

\newcommand{\cP}{\mathcal{P}}
\newcommand{\cM}{\mathcal{M}}
\newcommand{\Pcal}{\mathcal{P}}
\newcommand{\Lcal}{\mathcal{L}}
\newcommand{\Mcal}{\mathcal{M}}

\newcommand{\Acal}{\mathcal{A}}
\newcommand{\Ecal}{\mathcal{E}}
\newcommand{\Scal}{\mathcal{S}}
\newcommand{\tfrak}{t}
\newcommand{\Tcal}{T}
\newcommand{\Fcal}{\mathcal{F}}
\newcommand{\dfrak}{\mathfrak{d}}
\newcommand{\Bcal}{\mathcal{B}}

\DeclareMathOperator*{\argmin}{argmin}
\DeclareMathOperator{\clip}{clip}

\newcommand{\R}{\mathbb{R}}

\newtheorem{assumption}{Assumption}

\newtheorem{lemma}{Lemma}
\newtheorem{theorem}{Theorem}
\newtheorem{definition}{Definition}
\newtheorem{corollary}{Corollary}

\usepackage{bbm}
\renewcommand{\chi}{\mathbbm{1}}
\newcommand{\setting}{CB with User-triggered Supervision\xspace}
\newcommand{\setshort}{CBUS\xspace}

\newcommand{\ttc}{\texttt{TTC}\xspace}
\newcommand{\alg}{EFBO\xspace}
\newcommand{\alglong}{Explore First, Blend Optimally\xspace}
\newcommand{\Delb}{\bar\Delta}
\newcommand{\inner}[1]{\langle #1\rangle}

\newcommand{\tmax}{\ensuremath{t_{\max}}}

\usepackage[capitalize,noabbrev]{cleveref}

\makeatletter
\providecommand\theHALG@line{\thealgorithm.\arabic{ALG@line}}
\makeatother

\icmltitlerunning{Leveraging User-Triggered Supervision in Contextual Bandits}

\begin{document}

\twocolumn[
\icmltitle{Leveraging User-Triggered Supervision in Contextual Bandits}




\begin{icmlauthorlist}
\icmlauthor{Alekh Agarwal}{comp}
\icmlauthor{Claudio Gentile}{comp}
\icmlauthor{Teodor V. Marinov}{comp}
\end{icmlauthorlist}

\icmlaffiliation{comp}{Google Research}

\icmlcorrespondingauthor{Teodor V. Marinov}{tvmarinov@google.com}

\icmlkeywords{Machine Learning, ICML}

\vskip 0.3in
]



\printAffiliationsAndNotice{}  

\begin{abstract}
We study contextual bandit (CB) problems, where the user can sometimes respond with the best action in a given context. Such an interaction arises, for example, in text prediction or autocompletion settings, where a poor suggestion is simply ignored and the user enters the desired text instead. Crucially, this extra feedback is \emph{user-triggered} on only a subset of the contexts. We develop a new framework to leverage such signals, while being robust to their biased nature. We also augment standard CB algorithms to leverage the signal, and show improved regret guarantees for the resulting algorithms under a variety of conditions on the helpfulness of and bias inherent in this feedback. 
\end{abstract}



\section{Introduction}
\label{sec:intro}

Consider a learning agent for predicting the next word as a user composes a text document or an email. Such an agent can be pre-trained on an offline dataset of documents to predict the next word according to a language model, but it is often desirable to further improve the models for the task at hand, based on the data collected upon deployment. Such an improvement from logged data is not amenable to supervised learning, as we only observe whether a user liked the suggestions showed by the model, with no feedback on the quality of other actions. Consequently, a popular paradigm to model such settings is that of Contextual Bandits (CB), where the model is optimized to maximize a notion of reward, such as the likelihood of the predicted word being accepted by the user. The CB approach has in fact been successfully and broadly applied in online recommendation settings, owing to a natural fit of the learning paradigm.

However, in the example of next word prediction above, the standard CB model ignores important additional signals. When the user at hand does not accept the recommended word, they typically enter the desired word, which is akin to a supervised feedback on the best possible word in that scenario. How should we leverage such an extra modality of feedback along with the typical reward signal in CBs? While prior works have developed hybrid models such as learning with feedback graphs (e.g., \cite{10.5555/2986459.2986536,10.5555/3020652.3020671,doi:10.1137/140989455}) to capture a continuum between supervised and CB learning, such settings are not a natural fit here. A key challenge in the feedback structure is that the extra supervised signal is only available on a subset of the contexts, which are {\em chosen by the user} as some unknown function of the algorithm's recommended action. We term this novel learning setting 
\emph{\setting} (\setshort). In this paper, we develop theoretical frameworks and algorithms to address \setshort problems.

In addition to the supervision being user triggered, an additional challenge in the \setshort setting is that, unlike in learning with feedback graphs, the supervised feedback and the reward signal are not naturally available in the same units. For instance, in the next word prediction setting, a natural reward metric might be \emph{time-to-completion} (\ttc), that is, the time a user takes to enter a word (either accepting a recommended word or typing it manually). When the user does not accept the recommended word, they will enter a new word manually, and it is natural to expect that the \ttc would be minimized if this new word were recommended instead. Since we do not know the \ttc for any other word, this makes it challenging to reconcile the supervised feedback with the CB rewards. To overcome this issue, we develop a constrained optimization framework, where the learner seeks to optimize its CB reward while also trying to do well under the expected supervised learning error. The intuition is to guide the learner to a reasonable family of models using the supervised performance constraint, among which reward optimization can be fine-tuned for the performance metric that we eventually want to maximize.

Our work can be considered as part of the CB literature with constraints which has been extensively studied in several different settings. For a more careful discussion of these settings we refer the reader to Appendix~\ref{app:related_work}. Prior work can be roughly split into three categories. First is bandits with knapsacks where the additional constraint is modeled as a knapsack problem and the game ends when the knapsack constraint is exceeded~\citep{badanidiyuru2018bandits, tran2010epsilon, tran2012knapsack, ding2013multi, xia2015thompson, zhu2022pareto, agrawal2014bandits, wu2015algorithms, agrawal2016linear,sun2017safety, immorlica2022adversarial, sivakumar2022smoothed}. Second is conservative bandits where the player has to play a policy which is never much worse compared to a baseline~\citep{wu2016conservative, kazerouni2017conservative, garcelon2020improved, lin2022stochastic, garcelon2020conservative}. Perhaps closest to our work is that of the setting in which there exist two distributions, one over rewards for actions, and one over costs. The goal is to maximize the expected reward, while ensuring that the expected cost of the selected action is below a certain threshold~\citep{amani2019linear, moradipari2021safe, pacchiano2021stochastic}. Crucially none of these frameworks allow for observing the constrained only on an uncontrolled subset of the rounds, which is a key challenge of the \setshort setting.

\textbf{Our Contributions.} In addition to formalizing the \setshort framework for the learning settings of interest, our paper makes the following key contributions.
\begin{enumerate}
    \item \emph{Constrained formulation:} We propose a new constrained optimization approach for solving \setshort problems, where the objective encourages reward maximization and constraints capture fidelity to the supervised feedback. The constraints are enforced across all the rounds, independent of whether we observe the supervised feedback.
    \item \emph{Lower bound:} We show a fundamental tradeoff between the best attainable regret in terms of the bandit rewards and the supervised constraints. Informally, we show that the learner incurs an $\Omega(T^{2/3})$ regret on at least one of the expected reward or constraint violation, over $T$ rounds.
    \item \emph{Simple and optimal algorithm:} We develop an explore-first strategy (\alg) which performs initial exploration to gather a diverse dataset for both the CB rewards and the supervised feedback. We then solve the constrained optimization problem on this dataset using a saddle-point approach, and provide guarantees on the regret and constraint violation of \alg. The guarantees improve upon those for learning from supervised or CB signals alone, under an alignment condition on the two sources, and scale as $O(T^{2/3})$, matching the lower bound.
    \item \emph{Leveraging favorable distributions:} We develop an Exp4-based algorithm that can benefit from favorable conditions on the user, such as feedback from the user is only withheld if the selected action has small supervised learning error. This algorithm enjoys improved $O(\sqrt{T})$ regret, both for reward and constraint violation, allowing us to go beyond the lower bound by leveraging problem structure. We also design an active learning strategy to explicit helpful structures in the constraint function.
\end{enumerate}

\section{Problem Setting and a Lower Bound}
\label{sec:set-lb}
In this section, we formally describe the \setshort learning protocol, and also give a lower bound on the fundamental trade-off between the achievable regret on CB rewards and that on user supervision.

\subsection{The \setshort Problem Setting}
\label{sec:setting}

We are given a context space $\cX$ and an action space $[K]$ of size $K \geq 2$.
In the \setshort protocol, the learner observes some context $x_t \in \cX$ at time $t$, and has to choose an action $a_t \in [K]$. Upon choosing $a_t$, one of two things happen:
\begin{enumerate}
    \item The learner observes the reward $r_t \sim D_b(\cdot | x_t, a_t)$, $r_t \in [0,1]$, for the chosen action from the conditional reward distribution $D_b(\cdot | x_t, a_t)$, given the context $x_t$ and the action $a_t$ at hand, or 
    \item The learner observes $r_t = 0$ together with a special action $\atrue_t = \atrue(x_t)$, and has access to a \emph{surrogate loss} function $\Delta(a,a';x_t)$ for any $a$ relative to $a'$, given context $x_t$. The rounds $t$ on which $r_t = 0$ is observed {\em are not under the learner's} control (``user triggered"), and we define an indicator $\xi_t = 1$ to track these rounds. 
\end{enumerate}
Given a input tolerance $\epsilon > 0$, and a (finite) policy space $\Pi$ of functions $\pi(\cdot)$ mapping contexts to actions, and a distribution $D$ over $\cX$, we wish to solve the following policy optimization problem:
\begin{align}
        &\max_{\pi\in\Pi} \E_{x\sim D} \E_{D_b}[r | x,\pi(x)] \tag{Performance}\\
        &\mbox{s.t.}\quad \E [\Delta(\pi(x),\atrue(x);x)] \tag{Fidelity}\\
        &\qquad\quad\leq \min_{\pi' \in \Pi} \E [\Delta(\pi'(x),\atrue(x);x)] + \epsilon .
        \label{eq:obj}
\end{align}
In words, we would like to find a policy $\pi \in \Pi$ that maximizes the expected reward, subject to the constraint that, on average over the contexts, the amount by which the surrogate loss between the action selected by $\pi$ and the special action $\atrue$ exceeds the minimal expected surrogate loss achieved by policies in $\Pi$ by no more than $\epsilon$. Note that $\atrue(x)$ can be random, and the expectation in the constraint includes the randomness in both $x$ and $\atrue(x)$.
We call the expected reward our {\em performance} criterion, and the expected surrogate loss constraint our {\em fidelity} criterion.

We now illustrate how this formulation captures relevant practical scenarios.
\begin{example}[Next word prediction]
As a first motivating example, consider the next word prediction problem discussed in Section~\ref{sec:intro}.
The context $x_t$ consists of the preceding text, as well as any prior information on the user's writing style, demographics, etc. Feasible actions in a context $x_t$ might be plausible next words proposed by some base model, and the reward $r_t$ can be binary, based on the user accepting the suggested word, or more fine-grained such as \ttc. The latter might reward the learner more for correct predictions on longer words, for instance, than for common and short stop words. If the recommendation is not accepted (the learner observes $r_t = 0$) the word entered by the user provides $\atrue(x_t)$, and $\Delta(a, \atrue(x_t); x_t)$ can be a contextual measure of word similarity, such as distance in a word embedding space. The objective~\eqref{eq:obj} then incentivizes the maximization of the desired performance metric, while guaranteeing fidelity to the ground-truth signals provided by the user.
\label{ex:text}
\end{example}

\begin{example}[Rich in-session interaction]
As another example of \setshort, consider a user interacting with a recommendation system through multiple modes, like clicks, conversions, and textual queries. The goal of the recommendation system is to improve user experience by minimizing the time it takes for the user to find the information they are looking for. Each round $t$ is a user session. The user may start the session by entering some text (say a product they are interested in buying), the system may respond with a list of links to relevant products, then the user may react by either clicking on some product in the list or decide to refine their search by entering new and possibly more specific text. In this case, the context $x_t$ may encode the user's past behavior from previous sessions, as well as the initial query typed in during session $t$, the set of actions may include content which are relevant to this initial query, the reward $r_t$ may be some function of the value of a click or a conversion on one of the recommended items/products, while the fact that the initial recommendations are not accepted ($r_t= 0$) are witnessed by the extra text the user decides to type in. In this case, $\atrue(x_t)$ may encode the ``correct" product for $x_t$ as evinced by the new and more specific query the user enters.
Finally, $\Delta(a, a'; x_t)$ can be a contextual measure of pairwise similarity between items/products.
\label{ex:session}
\end{example}
A key challenge here is that the feedback $\atrue(x_t)$ is only observed on a subset of the rounds which are not controlled by the algorithm. Yet, the fidelity constraint seeks to enforce it in expectation over the full context distribution, and we are unable to correctly estimate this expectation using feedback only from the rounds where we observe $\atrue(x_t)$. For ease of presentation, we use $\super_t$ to denote the indicator of whether $\atrue(x_t)$ was observed at time $t$, and note that the distribution of $\super$ as a random variable depends both on the context $x$ and the learner's action $a$. 
We are going to measure the sub-optimality of any policy $\pi$ to the solution, $\pi^*$, of the problem in (\ref{eq:obj}) by the psuedo-regret\footnote
{
For simplicity we refer to the pseudo-regret as regret.
} 
over $T$ rounds of interactions with the environment incurred by $\pi$ to the objective and constraint respectively, defined as follows:
\begin{align*}
    \reg_r(\pi) &= \Bigl(\E[r(\pi^*(x), x)] - \E[r(\pi(x), x)]\Bigl)\\
    \reg_c(\pi) &= \Bigl(\E[\Delta(\pi(x), \atrue(x); x)]
     - \E[\Delta(\pi^*(x), \atrue(x); x)]\Bigl)~.
\end{align*}
For any distribution, $Q\in \Delta(\Pi)$, over the policies $\Pi$, we define $
\reg_r(Q) = \E_{\pi \sim Q}[\reg_r(\pi)]~
$, and $\reg_c(Q)$ in a similar manner.
Finally, for any algorithm $\Acal$ which produces a sequence of distributions $(Q_t)_{t\in[T]}$, we define 
\begin{align*}
&\reg_r(\Acal, T) = \sum_{t=1}^T \reg_r(Q_t)~,
\end{align*}
%
%
and define $\reg_c(\Acal, T)$ similarly by using $\Delta$ instead of $r$. The upper and lower regret bounds that we prove will all be in expectation with respect to the randomness in the algorithm as well, that is we show upper and lower bounds on $\E[\reg_r(\Acal, T)]$ and $\E[\reg_c(\Acal, T)]$.

\subsection{Revealing assumption and min-max rates}
\label{sec:lb}

In order to better understand our problem, the first thing to observe is that objective~\eqref{eq:obj} can be arbitrarily hard to achieve a good performance on, in the sense of simultaneously controlling both $\reg_r$ and $\reg_c$. This is due to the user-triggered nature of the supervised signal $\atrue(x)$. As an extreme case, suppose $\atrue(x)$ is never revealed by the user, even when the chosen actions are highly suboptimal under $\Delta$, then $\reg_c$ will clearly be $\Omega(T)$. However, this does not correspond to natural scenarios, since we expect the user not to accept bad recommendations, and hence there should typically be actions which lead to the revelation of $\atrue(x)$ in any context. Another common alternative is to simply omit a recommendation if we hope to elicit the ground-truth. We now make a concrete assumption to formalize this intuition and avoid trivial lower bounds.
\begin{assumption}[Revealing action]
\label{assm:revealing_action}
There exists a \emph{revealing} action $a_0\in\cA$ such that whenever the learner selects $a_0$ they get to observe $\bar a(x)$, that is, they get to observe the full feedback for the constraint given by $\Delta(\cdot,\bar a(x);x)$.
\end{assumption}
Note that the revealing action can be context dependent in general, so long as it is known, and all of our work is fully compatible with this generalization. We use a fixed revealing action $a_0$ solely for notational simplicity.

Even under the availability of $a_0$, the learner faces a more nuanced exploration dilemma. It can engage in natural exploration over $\Pi$ for optimizing rewards, and obtain incidental and biased observations of $\atrue(x)$, or occasionally choose $a_0$ to learn about the constraint. This sets up a potential trade-off between the two regrets $\reg_r$ and $\reg_c$, and we now give a fundamental characterization of the best achievable trade-off next.

\begin{theorem}[Lower bound]
\label{thm:lower_bound}
For any algorithm, $\Acal$, which has constraint regret at most $\E[\reg_c(\Acal, T)]$, there exists an instance on which the algorithm suffers reward regret 
$$
\E[\reg_r(\Acal, T)] = \Omega\left(\min\left(T\epsilon, \frac{T}{\sqrt{\E[\reg_c(\Acal, T)]}}\right)\right)~.
$$
\end{theorem}

We defer the construction of the problem instance and the proof of Theorem~\ref{thm:lower_bound} to Appendix~\ref{app:lower_bound}. The lower bound shows that in general it is not possible to achieve $O(\sqrt{T})$ regret for both the reward and the constraint under Assumption~\ref{assm:revealing_action}. We note that it may be possible to achieve $O(T^{2/3})$ regret simultaneously for the constraints and the reward (ignoring any dependence on the size of the action set and policy class). In general if the regret for the constraint is $O(T^{\alpha})$ then there exists an environment in which the algorithm incurs $\Omega(T^{1-\alpha/2})$ regret for the reward.

\section{A Simple and Optimal Algorithm}
\label{sec:explore-first}

To build intuition for the setting, we begin with an explore-first strategy which performs an initial exploration to separately learn about the rewards and the constraint. The exploration data is used to find a near-optimal solution to~\eqref{eq:obj}. While explore-first is statistically sub-optimal in an unconstrained scenario, this approach will be shown to match our lower bound in the constrained setting. We start with the algorithm and then present the regret guarantee.

\subsection{The \alglong Algorithm}

Given any $T_0 \leq T/2$, we might choose random actions for the first $T_0$ rounds and the revealing action $a_0$ for the subsequent $T_0$ rounds to form estimators for the reward and constraint violation for any policy $\pi\in\Pi$ as:
\begin{align}
    \label{eq:reward-cons}
    \Rhat(\pi) =& \frac{K}{T_0}\sum_{t=1}^{T_0} r_t\one(a_t = \pi(x_t))~,\\
    \regh_c(\pi) =& \frac{1}{T_0}\Bigl[\sum_{t=T_0+1}^{2T_0} \Delta_t(\pi(x_t)) - \min_{\pi'\in\Pi} \sum_{t=T_0+1}^{2T_0} \Delta_t(\pi'(x_t))\Bigl],\nonumber
\end{align}
where $\Delta_t$ is a shorthand for $\Delta(\cdot, \atrue_t; x_t)$. Then we might solve an empirical version of the objective~\eqref{eq:obj}, and use standard concentration arguments to guarantee good performance in terms of regret. However, this simple approach has a significant drawback. 

Suppose that $\Delta$ and the reward distribution $D_b$ are perfectly aligned, so that $\E[r | x,a] = 1-\E[\Delta(a, \atrue(x); x) | x,a]$ for all $x$ and $a$. Then choosing the revealing action $a_0$ reveals the rewards of all the actions, and hence we would expect guarantees compatible with supervised learning, where the suboptimality of the learned policy decays as $\sqrt{\ln|\Pi|/T_0}$ for both the objective and the constraint. On the other hand, the two distributions could be quite misaligned too, in which case the best reward suboptimality we can guarantee is $\sqrt{K\ln|\Pi|/T_0}$, incurring an additional $K$ factor due to the uniform exploration for learning the reward structure. Since we expect practical settings to be somewhere between these two extremes, we leverage ideas from~\citet{zhang2019warm} to take advantage of any available (unknown) alignment between the rewards and the constraints.

The algorithm, which we name \alglong(\alg) is presented in Algorithm~\ref{alg:explore_then_commit}. For an exploration parameter $T_0$, the algorithm chooses different types of exploration over $4T_0$ rounds. For the $2T_0$ rounds in $[T_0] \cup [3T_0+1,4T_0]$ we explore uniformly over the actions and record the rewards obtained. For the $2T_0$ rounds in $[T_0+1,3T_0]$ we choose the revealing action $a_0$ and observe $\atrue(x_t)$. Now we form the $\mu$-\emph{blended reward estimator}:
\begin{equation}
    \Rhat_\mu(\pi) = \mu\Rhat(\pi) + (1-\mu)\sum_{t=2T_0+1}^{3T_0} \frac{(1-\Delta_t(\pi(x_t)))}{T_0}.
    \label{eq:reward-mu}
\end{equation}
%
We note here that more generally, any other known function $g(\Delta)$ can be used to transform the constraints to be more compatible with rewards, in place of the choice $g(u) = 1-u$ used here. As long as the function takes bounded values, most of our results directly extend to this generalization.
We still use the same constraint estimator as in~(\ref{eq:reward-cons}) (so constraints and rewards using observations of $\Delta_t$ from disjoint rounds). Next, we need to optimize a constrained optimization with the objective $\Rhat_\mu(\pi)$ and constraint $\regh_c(\pi)\leq\epsilon$. In particular, we assume that we are given a class $\cM$ of candidate $\mu$-values, and find the best policy for each $\mu\in\cM$. Following prior works~(e.g., \citep{langford2007epoch, agarwal2014taming, agarwal2018reductions}), 
we only assume the ability to solve reward maximization problems over the policy class, which is needed even in the unconstrained case. We use a common primal-dual approach to solve the constrained problem by defining a Lagrangian for any $Q\in\Delta(\Pi)$ as:
\begin{equation}
\label{eq:lagrangian}
\small
    \begin{aligned}
        &\lagh_{\mu}(Q,\lambda) = \Rhat_\mu(Q) - \lambda \regh_c(Q),
    \end{aligned}
\end{equation}
where $\Rhat_\mu(Q)$ and $\regh_c(Q)$ are defined via expectations under policy distributions just like true rewards and regrets. Lines~7--9 
in Algorithm \ref{alg:explore_then_commit}
optimize the empirical saddle-point objective $$\max_{Q\in\Delta(\Pi)} \min_{\lambda \in [0,B]} \lagh(Q,\lambda).$$ The optimization uses the approach pioneered by~\citet{freund1996game} to interpret the objective as a two player zero-sum game, which is solved by alternating between a best response strategy for the policy player, and a no-regret strategy for the $\lambda$ player. The best response corresponds to finding the best policy under an appropriate reward definition (line 8), since all $\pi$-dependent terms in $\lagh(\pi,\lambda)$ are just functions of $\pi(x_t)$, and $\lagh(Q,\lambda)$ is optimized at a point mass on some policy $\pi\in\Pi$, due to the linearity in $Q$. We optimize over the scalar $\lambda$ using the Multiplicative Weight Updates algorithm (MWU)~\citep{arora2012multiplicative} together with a clipping operator (in line 9), which is a standard no-regret strategy for bounded subsets of the positive orthant. Alternating these steps for $S$ iterations yields an approximate solution for each fixed $\mu \in \Mcal$, denoted by $\hat Q_{\mu}$. Hence, we expect that all $\hat Q_\mu$ are feasible, but differ in their performance on the rewards. We then select the distribution $\hat Q_\mu$ with the highest empirical reward, evaluated on the second set of $T_0$ rewards collected by uniform exploration. That is our selected distribution is $\hat Q_{\hat \mu}$ where
\begin{align}
\label{eq:hatmu}
    \hat \mu = \argmax_{\mu\in \mathcal{M}} \frac{1}{T_0} \big\langle \hat Q_{\mu}, \sum_{t=3T_0+1}^{4T_0} \hat r_t(\cdot,x_t)\big\rangle~,
\end{align}
where $\hat r_t(a,x_t) = K r_t\one(a = \pi(x_t))$.
Finally, we play $\hat Q_{\hat \mu}$ for the remainder of the game.

\begin{algorithm}[t]
\caption{\alglong(\alg)}
\label{alg:explore_then_commit}
\begin{algorithmic}[1]
    \REQUIRE $4T_0$ rounds of exploration, $B$, $S$ parameters for constraints accuracy, set of mixture weights $\Mcal$
    \ENSURE Distribution $\hat Q_{\hat\mu}$ in $\Delta(\Pi)$
    \FOR{$t\in[T_0] \cup [3T_0+1,4T_0]$}
        \STATE Choose $a_t \sim Unif([K])$, observe reward $r_t(a_t,x_t)$
    \ENDFOR
    \FOR{$t\in[T_0+1,3T_0]$}
        \STATE Choose $a_t = a_0$ and observe $\Delta(\cdot, \bar a(x_t); x_t)$
    \ENDFOR
    \FOR{$\mu \in \Mcal$}
        \STATE $\lambda_{1,\mu} = \frac1B$, $Q_{1,\mu} = \displaystyle\argmax_{Q \in \Delta(\Pi)} \lagh_{\mu}(Q,\lambda_{1,\mu})$ \ \ \ (Eq.~(\ref{eq:lagrangian}))
        \FOR{$s\in [S]$} \label{algline:opt-begin}
            \STATE $Q_{s,\mu} = \argmax_{Q \in \Delta(\Pi)} \hat \Lcal_\mu(Q,\lambda_{s,\mu})$ \label{algline:best-response}
            \STATE $\lambda_{s+1,\mu} = \clip\left[MWU(\lambda_{s,\mu}, \hat \Lcal_\mu(Q_{s,\mu},\lambda_{s,\mu})) \vert B\right]$\label{algline:no-reg}
        \ENDFOR\label{algline:opt-end}
        \STATE $\hat Q_{\mu} = (Q_{1,\mu} + \ldots + Q_{S,\mu})/S$
    \ENDFOR
    \STATE \textbf{return} $\hat Q_{\hat \mu}$ \ \ \ (see Eq.~\ref{eq:hatmu})
\end{algorithmic}
\end{algorithm}


\subsection{Regret guarantee}
\label{sec:epsilon_greedy_regret}
We express our regret guarantees in terms of the degree of similarity between reward and constraint signals,
%
which is inspired by the work of \citet{zhang2019warm}.
\begin{definition}
\label{def:similarity}
A distribution $D_2$ is said to be $(\alpha,\mathfrak{d})$-similar to a distribution $D_1$ with respect to the tuple $(\Pi,\pi^\star)$ if
\begin{align*}
    \E_{D_2}&[r_2(\pi^\star(x),x)] - \E_{D_2}[r_2(\pi(x),x)]\\
    &\geq \alpha\Bigl(\E_{D_1}[r_1(\pi^\star(x),x)] - \E_{D_1}[r_1(\pi(x),x)]\Bigl) - \mathfrak{d}~.
\end{align*}
\end{definition}
In our setting we let $\E_{D_1}[r_1(\pi(x),x)] = \E[r(\pi(x),x)]$ and $\E_{D_2}[r_2(\pi(x),x)] = 1 - \E[\Delta(\pi(x),\bar a(x);x)]$, and use $\pi^\star$
as the solution of the problem in (\ref{eq:obj}).
Definition~\ref{def:similarity} essentially measures how well the full information component of the feedback, in the form of $1-\Delta$ is aligned with the bandit part of the reward, given by $\hat r_t$. The smaller $\dfrak$ is and the larger $\alpha$ is, the better the two distributions are aligned, which in turn will result in regret guarantees closer to the full information setting, in that the dependence on $K$ will be mild. 

We can now state the main theorem for this section. 
Before stating the regret bound we define
\begin{align}\label{eq:v_t(mu)}
V_{T_0}(\mu, v) 
&= 
2\sqrt{2 T_0(\mu^2 K + (1-\mu)^2v^2)\log(4|\Pi|T_0)}\notag\\
&\qquad + (\mu K + (1-\mu))\log(4|\Pi|T_0)~.
\end{align}
\begin{theorem}
\label{thm:etc_main}
Set in \alg\ the parameter values $S = \Omega(B T_0)$ and $B = T/T_0$. If the distribution over the constraints $\Delta(\cdot, \bar a(x); x)$ is $(\alpha,\dfrak)$-similar to $D_b$, the expected reward regret $\E[\reg_r(\hat Q_{\hat \mu})]$ is bounded by
\begin{align*}
     O\Bigg(\hspace{-0.03in}\sqrt{\frac{K\log(T_0|\mathcal{M}|)}{T_0}}
    + \hspace{-0.03in}\min_{\mu \in \mathcal{M}} \frac{\frac{2V_{T_0}(\mu,1)}{T_0} + (1-\mu)\mathfrak{d}}{\mu + \alpha(1-\mu)} + \frac{T_0}{T}\hspace{-0.03in}\Bigg).
\end{align*}
Further, the expected regret to the constraint is bounded as
\begin{align*}
    \E[\reg_c(\hat Q_{\hat \mu})] \leq \epsilon + O\left(\sqrt{\frac{\log(T_0|\Pi|)}{T_0}} + \frac{T_0}{T}\right).
\end{align*}
\end{theorem}
Note, that we can show the above regret bounds hold with high probability as well. In practice, we choose the class $\Mcal$ to be relatively small (constant or $|\Mcal| = O(\log(T))$), so for the remainder of the discussion we treat $\log(|\Mcal|)$ as a lower order term.

We prove Theorem~\ref{thm:etc_main} in Appendix~\ref{app:proof-etc}. To interpret the result, we examine different regimes of distributional similarity. 

\textbf{Minimax optimality.} Choosing $T_0 = \Theta(T^{2/3})$ above, the expected reward regret 
satisfies
\begin{align*}
    &T\E[\reg_r(\hat Q_{\hat\mu})] \leq O\Big(T^{2/3}\sqrt{K\log(T)}\\ 
    &+\min_{\mu \in \Mcal} \frac{T^{2/3}\sqrt{(\mu^2 K + (1-\mu)^2)\log(|\Pi|T)} + T(1-\mu)\dfrak}{\mu + \alpha(1-\mu)}\Big),
\end{align*}
while 
$$
T\E[\reg_c(\hat Q_{\hat \mu})] = O\big(T\epsilon + T^{2/3}\sqrt{\log(T |\Pi|)}\big)~.
$$
In terms of the scaling with $T$, this bound is minimax optimal due to the lower bound of Theorem~\ref{thm:lower_bound}. We note that this is in contrast with the suboptimality of explore-first in the unconstrained setting, and a consequence of the trade-off between constraint and reward exploration inherent in our framework. However, the relatively crude setting of $T_0$ here does not recover the best bound using explore-first even in the unconstrained setting (in $K$ and $\ln|\Pi|$ scaling). For a finer grained understanding, we now make distributional similarity assumptions, under which we can make better choices of $T_0$ as a function of the ideal $\mu$ value, and obtain sharper bounds. We note that the inability to depend on the best $\mu$ in hindsight for $T_0$ is akin to the difficulty of choosing hyperparameters in model selection~\citep{marinov2021pareto, zhu2022pareto}.

\textbf{Well-aligned signals.} In this case, we assume $\alpha = 1$ and $\dfrak = O(T^{-1/2})$. The RHS of Theorem~\ref{thm:etc_main} is then minimized for $\mu = O(1/\sqrt{K})$, and $T\E[\reg_r(\hat Q_{\hat\mu})]$ is at most
$$
O(T\sqrt{K\log(T|\Mcal|)/T_0} + T\sqrt{\log(T|\Pi|)/T_0} + T_0)~.
$$ 
Choosing $T_0 = \Theta(T^{2/3}(K\log(T)\lor \log(|\Pi|T))^{1/3})$ optimally further implies 
$$
T\E[\reg_r(\hat Q_{\hat\mu})] = O\left(T^{2/3}(K\log(T)\land \log(|\Pi|T))^{1/3}\right)~,
$$ 
that is, we achieve a bound which decouples the bandit part of the regret, $K$, from the policy class part $\log(|\Pi|)$. This is analogous to the benefit of similarity in~\citet{zhang2019warm}. The constraint violation regret admits the same bound.

\textbf{Mis-aligned signals.} On the other extreme, when $\dfrak = \Omega(1)$, we take $\mu = 1$ and set $T_0 = T^{2/3}(K\log(T|\Pi|))^{1/3}$. This gives a bound consistent with the standard CB setting, that is
$$
T\E[\reg_r(\hat Q_{\hat\mu})] \leq
O(T^{2/3}\big(K\log(T|\Pi|))^{1/3}\big)~.
$$
Finally, we address the size of $\Mcal$. As discussed, the favorable case is when $\dfrak \approx 0$ and thus $\mu = O(1/\sqrt{K})$. Hence it is sufficient to take 
$$
\Mcal = \{1 - 1/2^n, 1/K + 1/2^n : n \leq \log(T)\}
$$ 
(see Lemma~\ref{lem:mcal_choice} in the Appendix~\ref{app:proof-etc} for details). 

\section{Improving Regret under Favorable Conditions}
We now present a high-level algorithmic framework which maintains the worst-case statistical optimality of \alg, while allowing the possibility of stronger results under favorable problem structures, such as a relationship between the user decision to provide the supervision $\atrue(x)$.
Since the algorithm is more complex, we first provide the high-level structure, before moving to concrete instantiations of some components later in the section. 
The algorithm is a version of a corralling algorithm~\citep{agarwal2017corralling} applied to an adaptation of the classical Exp4 algorithm~\citep{auer2002nonstochastic}. At any round $t$, our adapted Exp4 incorporates an arbitrary constraint estimator $\bar \Delta_t$ for $\Delta(a, \atrue(x_t); x_t)$. The estimator is used as part of the reward signal, similarly to how the rewards are constructed in Algorithm~\ref{alg:explore_then_commit}. Secondly, the estimator is used to maintain approximately feasible policies $\Pi_t\subseteq \Pi$, as a proxy for policies feasible for~\eqref{eq:obj}. 

A formal description of the modified Exp4 algorithm can be found in Equation~\ref{eq:ftrl_exp4} in Appendix~\ref{app:ftrl_proofs}. Since the Exp4 update only works for a fixed combination of $\Delta_t$ and reward $r_t$ we further use model selection over a $\mu$ parameter used to blend rewards in a similar way as \alg, through corralling the Exp4 algorithms, each corresponding to a single $\mu$. Formally this is achieved by running a version of the Hedged FTRL corralling algorithm described in \citep{foster2020adapting, marinov2021pareto}. Pseudo-code for this algorithm is in Algorithm~\ref{alg:hedged_ftrl}. The algorithm also includes an indicator $Z_t$ as some (adaptively chosen) rounds might be needed to form the constraint estimator $\bar\Delta_t$ in the subsequent instantiations. On these rounds with $Z_t=1$, Exp4 does not update its internal state (lines 9-10) . We set $M = O(\log(T))$ and each base algorithm uses Equation~\ref{eq:ftrl_exp4} with $\mu \in \{1 - 1/2^n, 1/K + 1/2^n : n \leq \log(T)\}$, same as in Algorithm~\ref{alg:explore_then_commit}. The main regret bound can be found in Theorem~\ref{thm:hedged_bound} in Appendix~\ref{app:ftrl_proofs}.
\begin{algorithm}
\caption{Corralling Exp4 with constraints}
\label{alg:hedged_ftrl}
\begin{algorithmic}[1]
    \REQUIRE $(Base_m)_{m=1}^M$
    \STATE Initialize $P_1$ to be uniform distribution over $(Base_m)_{m=1}^M$ base algorithms.
    \STATE Initialize constraint proxy $\bar \Delta_1$, and base algorithms $(Base_m)_{m=1}^M$.
    \FOR{$t=1,\ldots, T$}
        \STATE Receive context $x_t$, compute set of feasible policies $\Pi_t\subseteq \Pi_{t-1}$, sample $Z_t$.
        \IF{$Z_t = 0$}
            \STATE Sample base algorithm $m_t \sim P_t$ and play according to policy, $\pi_t$, selected by $Base_{m_t}$.
            \STATE Observe loss $r_t(\pi_t(x_t);x_t)$ and $\bar\Delta_t(\cdot;x_t)$.
        \ELSE
            \STATE Play revealing action $a_0$, observe $\Delta(\cdot, \atrue(x_t); x_t)$.
        \ENDIF
        \STATE Update $P_{t+1}$ using Hedged-FTRL (\citet{marinov2021pareto} Algorithm~1). 
        \STATE Send feedback $r_{t,m} = \chi(m_t = m)/P_{t.m}, \bar \Delta_t$ to $m$-th base algorithm.
        \STATE Base algorithms update their state as per Eq.~(\ref{eq:ftrl_exp4}).
    \ENDFOR
\end{algorithmic}
\end{algorithm}

Next, we illustrate two instantiations for $\bar \Delta_t$ and $\Pi_t$, along with concrete theoretical guarantees. All results of this section are derived from a general result proved in Theorem~\ref{thm:hedged_bound} in Appendix~\ref{app:ftrl_proofs}. The first is based on the assumption that the supervision from the user is triggered by the choice of a significantly suboptimal action under the CB rewards, so that the lack of supervision is an implicit signal about the chosen action being fairly good in terms of reward. The second approach is based on active learning to adaptively learn the mapping $x\to \atrue(x)$ and use this mapping to induce the constraints on all points. In both settings we make the following mild assumption on $\Delta$.
\begin{assumption}
\label{assm:delta_assms}
$\Delta$ is symmetric for any $x\in\cX$, that is $\Delta(a,a';x) = \Delta(a',a; x)$ and further it satisfies a triangle inequality, that is $\Delta(a, b; x) \leq \Delta(a, a'; x) + \Delta(a', b; x)$.
\end{assumption}

For instance, the assumption holds if $\Delta(a, a'; x)$ is a distance between $a$ and $a'$ in some ($x$-dependent) embedding.

\subsection{Suboptimality-triggered supervision}
\label{sec:delta_estimators}
We now make the following assumption on when the supervised feedback $\atrue(x)$ is received.
\begin{assumption}[Suboptimality-triggered supervision]
\label{assm:accepting_constr}
At any round $t$, if the user does not reveal $\bar a(x_t)$ (i.e. $\xi_t = 0$), then it holds that $\Delta(a_t, \bar a(x_t); x_t) \leq \nu$.
\end{assumption}
This assumption is 
natural when the user behaves in a non-malicious way. Indeed, we expect that if the user accepts the learner's recommendation, the recommendation can not be much worse than what the user would have specified themselves.
Using the above assumptions we can construct the following simple constraint estimator.
\paragraph{A biased constraint estimator.} Let us define the following estimator for the true constraint: 
\begin{align*}
    \hat \Delta_t(\pi(x_t); x_t)
    &=  (1-\xi_t) \Delta(\pi(x_t), a_t; x_t)\\
    &\qquad+ \xi_t\Delta(\pi(x_t), \bar a(x_t); x_t)~,
\end{align*}
where we recall that $\xi_t = 1$ if the user reveals $\atrue(x_t)$. Clearly $|\hat\Delta_t(\pi(x_t); x_t) - \Delta(\pi(x_t), \bar a(x_t); x_t)| \leq \nu, \forall \pi \in \Pi$ under Assumption~\ref{assm:accepting_constr}, that is $\hat\Delta_t$ is a $\nu$-biased estimator for $\Delta$. Furthermore, it has a variance bounded by $1$, since $0\leq\Delta(a, a'; x)\leq 1$. Consequently, 
we can use Lemma~\ref{lem:shrinking_pol_set2} in Appendix~\ref{app:ftrl_proofs} to construct $\Pi_t$ as follows. Let $r_t = 2\nu  + 4\sqrt{2\frac{\log(T|\Pi|/\delta)}{t}}$, and set $\Pi_1 = \Pi$,
\begin{equation}
\begin{aligned}
    \label{eq:pi_t_bias1}
    \Pi_{t+1} = \Bigg\{&\pi \in \Pi_t : \frac{1}{t}\sum_{s=1}^t \hat \Delta_s(\pi(x_s); x_s)\\
    \leq &\min_{\pi \in \Pi_t} \frac{1}{t}\sum_{s=1}^t \hat \Delta_s(\pi(x_s); x_s) + \epsilon + r_t\Bigg\}.
\end{aligned}
\end{equation}
This construction ensures that all policies in $\Pi_t$ are only $O(r_t)$-suboptimal to the constraint.
We immediately obtain the following corollary of Theorem~\ref{thm:hedged_bound}.
Let
\begin{align*}
    \phi(\mu, v_m, T, \dfrak) = &\frac{(\mu^2 K + (1-\mu)^2v^2_m)\sqrt{T\log(|\Pi|)\log(T)}}{\mu + \alpha(1-\mu)}\\
        &\qquad\qquad+ \frac{T(1-\mu)\dfrak}{\mu + \alpha(1-\mu)}~.
\end{align*}
\begin{theorem}
\label{thm:biased_constraint}
Assume that the distribution over constraints $\Delta(\cdot, \bar a(x); x)$ is $(\alpha, \dfrak)$-similar to the distribution over rewards $r(\cdot, x)$ with respect to $(\Pi, \pi^*)$. Algorithm~\ref{alg:hedged_ftrl} invoked with $Z_t\equiv 0$, $(\Pi_t)_{t\in[T]}$ as in Eq.~\ref{eq:pi_t_bias1} and $\bar\Delta_t = \hat\Delta_t$ satisfies 
$$
\E[\reg_r(\Acal, T)] \leq \min_{\mu\in[0,1]} \phi(\mu, 1, T, \dfrak+\nu)~,\mbox{and}
$$
$$
\frac{\E[\reg_c(\Acal,T)]}{T} \leq \epsilon + 4\nu + 8\sqrt{2\frac{\log(T|\Pi|)}{T}}~.
$$
\end{theorem}
We note that Theorem~\ref{thm:biased_constraint} does not require Assumption~\ref{assm:revealing_action}. 

\textbf{Better bounds for small $\nu$}. When $\nu = O(1/\sqrt{T})$, Theorem~\ref{thm:biased_constraint} yields an $O(\sqrt{T})$ bound for both rewards and constraints. However, the regret to the constraint can be as large as $\Omega(\nu T)$ in the worst case, due to the bias in $\hat \Delta$. We can further improve the robustness of this estimator using a doubly robust approach, which we describe next.
\paragraph{Doubly robust estimator.}
Consider choosing the revealing action $a_0$ with probability $\gamma_t$ at round $t$ (i.e., $Z_t=1$ with probability $\gamma_t$). To obtain a better bias-variance tradeoff than the constraint estimator above, we consider a doubly-robust approach~\citep{robins1994estimation,dudik2014doubly}: 
$$
\bar \Delta_t(a; x_t) =  \hat \Delta_t(a; x_t) + Z_t\frac{(\Delta(a, \bar a(x_t); x_t) - \hat \Delta_t(a; x_t) )}{\gamma_t}.
$$ 
We note the distinction between $Z_t$ and $\xi_t$ here. $\xi_t$ is 1 for all rounds where $\atrue(x_t)$ is observed, irrespective of whether the chosen action was $a_0$ or some other action, while $Z_t=1$ only on the rounds where we choose $a_0$ intentionally, to avoid bias in the user's revelation of $\atrue(x_t)$ in response to the other actions. Due to this, the doubly robust estimator is unbiased and has variance bounded by $2+2\nu^2/\gamma_t$. Let 
\begin{center}
\(
U_t(\delta, v) =  4\sqrt{\frac{(1\lor \nu T^{1/4})\log(T|\Pi|/\delta)}{t}} + 4\frac{(T^{1/4})\log(T|\Pi|/\delta)}{t}
\) 
\end{center}
(see Lemma~\ref{lem:doubly_robust} in Appendix~\ref{app:delta_estimators}). 
In a similar way to Equation~\ref{eq:pi_t_bias1} we can construct the following nearly feasible policy sets,
\begin{equation}
\label{eq:pi_t_bias2}
\begin{aligned}
    \Pi_{t+1} &= \Bigg\{\pi \in \Pi_t : \frac{1}{t}\sum_{s=1}^t \bar \Delta_s(\pi(x_s); x_s)\\
    \leq &\min_{\pi \in \Pi_t} \frac{1}{t}\sum_{s=1}^t \bar \Delta_s(\pi(x_s); x_s) + \epsilon + 4U_t(\delta,\nu)\Bigg\}~.
\end{aligned}
\end{equation}
Setting $\gamma_t = \frac{\nu}{T^{1/4}}$ allows us to show the following result.
\begin{theorem}
\label{thm:biased_constraint2}
Assume that the distribution over constraints $\Delta(\cdot, \bar a(x); x)$ is $(\alpha, \dfrak)$-similar to the distribution over rewards $r(\cdot, x)$ with respect to $(\Pi, \pi^*)$. Algorithm~\ref{alg:hedged_ftrl} invoked with $Z_t= Ber(\gamma_t)$, $(\Pi_t)_{t\in[T]}$ defined in Eq.~\ref{eq:pi_t_bias2} satisfies 
$$
\E[\reg_r(\Acal, T)]= O(\min_{\mu \in [0,1]} \phi(\mu, 1\lor \nu T^{1/4}, T ,\dfrak))~,\mbox{and}.
$$
$$
\frac{\E[\reg_c(\Acal, T)]}{T} \leq \epsilon + O\left(\sqrt{\frac{\nu}{T^{3/4}}\log(T|\Pi|)} + \frac{\log^2(T|\Pi|)}{T^{3/4}}\right)~.
$$
\end{theorem}
\textbf{Better bounds for small $\nu$.} Theorem~\ref{thm:biased_constraint2} implies that as long as $\nu = O(1/T^{1/4})$ the instance of Algorithm~\ref{alg:hedged_ftrl} will incur only $O(\sqrt{T})$ regret (ignoring other multiplicative factors) to both the reward and constraint. This improves upon Theorem~\ref{thm:biased_constraint} by expanding the range of $\nu$ for the improved rate, at the cost of requiring Assumption~\ref{assm:revealing_action}. As with Theorems~\ref{thm:etc_main} and~\ref{thm:biased_constraint2}, we retain the ability to leverage distributional similarity in rewards and constraints.

\textbf{Robustness to large $\nu$.} When $\nu$ becomes too large, $\nu = \omega(1/T^{7/24})$, the regret bound in Theorem~\ref{thm:biased_constraint2} becomes asymptotically worse compared to that of Theorem~\ref{thm:etc_main}. This is because in this setting of $\nu$, $\gamma_t = \omega(1/T^{1/3})$ and the algorithm incurs large regret due to sampling $a_0$ too often. To correct this minor problem, we can additionally enforce $Z_t = 0$ for any $t\geq \tmax$, where $\tmax$ is the smallest round at which $\sum_{t=1}^{\tmax} Z_t \geq \Omega(T^{2/3})$. It is possible to show that in this case $\frac{1}{t}\sum_{t=1}^{\tmax} \bar \Delta_t$ will have similar statistical properties to the estimator of $\Delta$ in Section~\ref{sec:explore-first}. In particular this modification yields a regret bound (in terms of\ $T$) for Algorithm~\ref{alg:hedged_ftrl} of $O(T^{2/3})$ both for the reward and constraint, while retaining the $O(\sqrt{T})$ improvement for small $\nu$.

Note, that for both the biased estimator and the doubly-robust unbiased estimator we require knowledge of $\nu$ to be able to correctly instantiate $\bar \Delta_t$ and construct $\Pi_t$. Making these algorithms adaptive to the knowledge of $\nu$ is an important direction for future research. Our final approach does not require such knowledge of hyper-parameters and is inspired by the active-learning literature.

\subsection{An active learning approach}
\label{sec:active}
Now we consider a strategy for constraint estimation, where we use active learning to estimate $x\to \atrue(x)$ using policies in $\Pi$. The resulting optimization problem, however, is slightly different and the guarantees we get are not directly comparable to Theorems~\ref{thm:biased_constraint} and~\ref{thm:biased_constraint2}. We first define the query rule and sets $\Pi_t$. Set $\Pi_1 = \Pi$ and $r_t = 4\sqrt{\frac{2\log(|\Pi|/\delta)}{t}}$, and $\Scal(\pi, t) = \sum_{s=1}^t Z_s\Delta(\pi(x_s),\bar a(x_s); x_s)$. Define $\widehat \pi_t = \argmin_{\pi \in \Pi_t} \Scal(\pi, t)$ and
\begin{align}
    \label{eq:active_learning}
        \Pi_{t+1} &= \big\{\pi \in \Pi_t : \Scal(\pi,t) \leq \Scal(\widehat\pi_t, t) + (2\epsilon + 3r_t)t\big\}\nonumber\\
        Z_{t+1} &= \chi\Big(\exists\pi, \pi' \in \Pi_{t+1}: \Delta(\pi(x_{t+1}), \pi'(x_{t+1}); x_{t+1}))\nonumber\\
        &\qquad\qquad \geq \epsilon + r_{t+1}/2\Big).
\end{align}
The definition of $\Pi_t$ does not differ too much from the one using the biased estimator of $\Delta$ in the previous section, however, the query rule has now changed from a uniform exploration one to an active learning one. The rule states that the revealing action is played only when there exist at least two policies which have large disagreement with respect to $\Delta$ and have not yet been eliminated as infeasible. Under a Masssart-like noise condition on the constraint~\citep{massart2006risk} it is possible to show that $Z_t = 1$ only for polylog$(T)$ rounds. Let $\bar \pi = \argmin_{\pi \in \Pi} \E[\Delta(\pi(x), \bar a(x); x)]$. We state the desired noise condition below.

\begin{assumption}[Low noise in constraints] 
The constraint function $\Delta$ satsifies a low noise condition with margin $\tau$ if for all $x$ and $a \ne \bar \pi(x)$, we have $\Delta(a, \bar\pi(x); x) \geq \epsilon + \tau$.
\label{assm:constraint-massart}
\end{assumption}
The assumption is a natural modification of Massart's low noise condition to the problem of minimizing $\Delta(a, \cdot, \cdot)$ w.r.t. $a$, and similar assumptions have been used in active learning for cost-sensitive classification in~\citet{krishnamurthy2017active}. Intuitively, the assumption posits that every suboptimal action in terms of constraints has a lower bounded gap to $\bar\pi$'s action. In Appendix~\ref{app:active_learning}, we state a more general condition under which our results hold, but give the simpler condition here for ease of interpretability. 
\begin{theorem}
\label{thm:active_learning}
Assume that the distribution over $\Delta(\cdot, \bar\pi(x); x)$ is $(\alpha, \dfrak)$-similar to the reward distribution. Under Assumption~\ref{assm:constraint-massart}, the regret of Algorithm~\ref{alg:hedged_ftrl} invoked with $Z_t$ and $\Pi_t$ defined as in Equation~\ref{eq:active_learning} satisfies
\begin{align*}
    \frac{\E[\reg_r(\Acal, T)]}{T}\leq \frac{\log(T|\Pi|)}{T\tau^2}+O\Big(\min_{\mu\in[0,1]} \phi\big(\mu, 1, T, \dfrak + \epsilon\big)\Big)~,
\end{align*}
$$
\frac{\E[\reg_c(\Acal, T)]}{T} \leq 3\epsilon + O(\sqrt{\log(T|\Pi|)/T})~.
$$
\end{theorem}
We note that the constraint violation part of the regret has a constant multiplicative factor in front of $\epsilon$. This is due to the fact that the algorithm does not try to directly approximate $\Delta$. Further, note that the $(\alpha,\dfrak)$-similarity is stated in terms of $\Delta(\cdot, \bar \pi(x); x)$ rather than $\Delta(\cdot, \bar a(x); x)$, which is again due to the same reason. In fact, the active learning based algorithm might never have an accurate estimator of $\Delta$.

In terms of rates, we incur an $O(\sqrt{T})$ regret in both rewards and constraints modulo the caveat above, and noting that the constraint threshold $\epsilon$ also appears in the distributional bias term in the rewards regret. As a result, the guarantees here are generally incomparable with the previous results, but nevertheless useful for leveraging a problem structure complementary to our previous conditions.

Finally, we note that the noise condition in 
Assumption \ref{assm:constraint-massart}
can be replaced by a milder Tsybakov-like noise condition. More details and proofs of Theorem~\ref{thm:active_learning} can be found in Appendix~\ref{app:active_learning}. Note that Theorem~\ref{thm:active_learning} does not have meaningful guarantees if Assumption~\ref{assm:constraint-massart} fails to hold, however, a modification similar to the one discussed after Theorem~\ref{thm:biased_constraint2} can be implemented to again guarantee a $O(T^{2/3})$ regret bound for both the reward and constraint. 

\section{Discussion}
This paper initiates a theoretical investigation 
of CB problems where the learner observes extra supervised signals produced only on a subset of contexts/time steps which are not under the agent's control (``user triggered"), a practically prevalent scenario. The key challenge we overcome is the biased nature of these observations. We believe that the constrained learning and reward-blending framework used here is a flexible way to capture potentially biased signals which arise in practical deployment of CB algorithms. 

Looking ahead, there are important questions of robustness to violation of our assumptions, such as Assumption~\ref{assm:accepting_constr} which are not addressed here. Developing algorithms to leverage such favorable conditions while maintaining computational efficiency is another challenge. More broadly, it would be interesting to validate the assumptions developed here, or discover alternatives, through practical studies of user behavior in the motivating examples underlying our work. Addressing such questions is paramount to improving the sample-efficiency of CB algorithms in practice, and make them applicable in broader settings.

\bibliography{refs}
\bibliographystyle{icml2023}

\newpage
\appendix
\onecolumn
\section{Related work}
\label{app:related_work}
The problem of CB with constraints has already been studied in several different settings which we now outline. The \emph{Bandits with Knapsacks}~\citep{badanidiyuru2018bandits} problem is a version of the standard bandit problem, however, at every round the player also observes a cost vector $c_t \in \mathbb{R}^K$. The goal of the player is to maximize their cumulative reward, however, the bandit game ends whenever the total cost of any arm $i$ exceeds a predetermined budget $B$, that is the game ends at the smallest round $\tau$ where there exists $i \in [K]$ s.t. $\sum_{t=1}^\tau c_{t,i} \geq B$. The comparator in the regret bound is the best strategy with hindsight knowledge of the problem dependent parameters such as the reward distribution and the cumulative cost of all actions. There is a wide variety of modifications to the above problem studied in~\cite{tran2010epsilon, tran2012knapsack, ding2013multi, xia2015thompson, zhu2022pareto}, including the extension to general convex constraints and concave rewards~\citep{agrawal2014bandits} and the CB setting~\citep{agrawal2014bandits, wu2015algorithms, agrawal2016linear}. The problem has also been studied in the adversarial setting~\citep{sun2017safety, immorlica2022adversarial, sivakumar2022smoothed}.

Bandits with a base-line or \emph{conservative bandits}~\citep{wu2016conservative} is a different problem in which the player is required to perform no worse than the cumulative reward of a base-line strategy during every round of the game. This setting has been extended to CBs~\citep{kazerouni2017conservative, garcelon2020improved, lin2022stochastic} and Reinforcement learning~\citep{garcelon2020conservative}. For a more careful discussion on the above settings we refer the reader to \cite{lu2021bandit}.

Perhaps closest to our work is that of the setting in which there exist two distributions one over rewards for actions and one over costs. The goal is to maximize the expected reward, while ensuring that the expected cost of the selected action is below a certain threshold. The cost requirement can either be enforced with high probability over the rounds~\citep{amani2019linear, moradipari2021safe} or in expectation~\citep{pacchiano2021stochastic}. All of~\citep{amani2019linear, moradipari2021safe, pacchiano2021stochastic} work in the linear CB setting. Further, in their work it is assumed that the cost signal is observed in every round. Our work is set in the general CB setting and the cost/constraint signal might rarely be observed throughout the game. This respectively leads to a different min-max rate for the regret of the game we consider as compared to prior work. 

\section{Proofs from Section \ref{sec:lb}} \label{app:lower_bound}
\subsection{Proof of Theorem \ref{thm:lower_bound}}
\begin{proof}
We first define the specific learning problem (``environment") and then the strategy of the user.

\noindent{\bf Environment.} The action space $\mathcal{A} = \{a_{-1},a_{1}\}$ consists of two actions. The context space is $\mathcal{X} = \{\pm 1\}^k$. The policy space is $\Pi = \{\pi_1,\pi_2\}$ with $\pi_1(x) = a_{sgn(x)}$ and $\pi_2(x) = a_{-sgn(x)}$, where $sgn(x) = \prod_{i=1}^k x_i$. The distribution over contexts is uniform and the rewards are setup so that $\mathbb{E}[r|x,\pi_1(x)] \geq \mathbb{E}[r | x,\pi_2(x)]+c$ for some constant $c \gg 0$. Further, define the loss function for the constraints 
\begin{align*}
    \Delta(a,a';x) = 
    \begin{cases}
        0 \text{ if } sgn(x) = 1\\
        \chi(a\neq a') \text{ if } sgn(x) = -1
    \end{cases}
\end{align*}
where $\chi$ is the characteristic function.

\noindent{\bf Strategy of the user.} We define two strategies of the user between which we have to distinguish to determine if $\pi_1$ is feasible. Under strategy $\mathcal{S}_1$ the user selects
\begin{align*}
    \bar a(x) = 
    \begin{cases}
        a_{sgn(x)} \text{ with probability } \frac{1}{2}\\
        a_{-sgn(x)} \text{ with probability } \frac{1}{2}.
    \end{cases}
\end{align*}
Under $\mathcal{S}_1$ it holds that $\mathbb{E}_{\mathcal{S}_1}[\Delta(\pi_2(x),\bar a(x); x)] = \mathbb{E}_{\mathcal{S}_1}[\Delta(\pi_1(x),\bar a(x); x)] = \frac{1}{2}$. 
Under strategy $\mathcal{S}_2$ the user selects
\begin{align*}
    \bar a(x) = 
    \begin{cases}
        a_{sgn(x)} \text{ with probability } \frac{1}{2} - \gamma\\
        a_{-sgn(x)} \text{ with probability } \frac{1}{2} + \gamma.
    \end{cases}
\end{align*}
Under strategy $\mathcal{S}_2$ it holds that
\begin{align*}
   \mathbb{E}_{\mathcal{S}_2}[\Delta(\pi_2(x),\bar a(x); x)] &= \frac{1}{2} - \gamma\\
   \mathbb{E}_{\mathcal{S}_2}[\Delta(\pi_1(x),\bar a(x); x)] &= \frac{1}{2} + \gamma.
\end{align*}
Let $\mathbb{P}_1$ and $\mathbb{P}_2$ be the measures induced after $T$ interactions under strategy $\mathcal{S}_1$ and $\mathcal{S}_2$ respectively. Define $\mathbb{P}_{i,t} = \mathbb{P}_{i,t}(\cdot|\{x_s,\bar a(x_s),\Delta(a_s,\bar a(x_s);x_s)\}_{s=1}^{t-1})$ as the conditional measure generated by the first $t-1$ observations under strategy $i$. The chain rule for relative entropy implies
\begin{align*}
    KL(\mathbb{P}_1||\mathbb{P}_2) = \sum_{t=1}^T \mathbb{E}_{\mathbb{P}_1} KL(\mathbb{P}_{1,t} || \mathbb{P}_{2,t}) \leq 2\sum_{t=1}^T \gamma^2 \mathbb{E}_{\mathbb{P}_1}\chi(\pi_t = \pi_2) = 2\gamma^2 \mathbb{E}_{\mathbb{P}_1} N_{\pi_2}(T),
\end{align*}
where $N_{\pi_2}(T)$ denotes the number of times that $\pi_2$ has been played in the first $T$ rounds of the game. In the above derivation the first inequality holds because under the event $\pi_t = \pi_2$ the KL divergence between the conditional measures is the KL divergence between two Bernoulli r.v.'s with parameter $\frac{1}{2}$ and $\frac{1}{2}\pm \gamma$. By Pinsker's inequality we have $\mathbb{E}_{\mathbb{P}_2} N_{\pi_2}(T) - \mathbb{E}_{\mathbb{P}_1} N_{\pi_2}(T) \leq T\gamma\sqrt{\mathbb{E}_{\mathbb{P}_1}N_{\pi_2}(T)}$.
Let $\E\reg_{1}$ denote the expected regret under strategy $\mathcal{S}_1$ for the rewards part of the objective, and let $ \E\bar\reg_2$ denote the regret of the constraints part of the objective under strategy $\mathcal{S}_2$. Then we have $\mathbb{E}_{\mathbb{P}_2}N_{\pi_2}(T) = \E\reg_1(T)/c$. Further, by combining this observation with the bound from Pinsker's inequality it holds that
\begin{align*}
    \E\bar\reg_2(T) &\geq \gamma(T-\mathbb{E}_{\mathbb{P}_2} N_{\pi_2}(T)) \geq \gamma(T - \E\reg_1(T)/c - T\gamma\sqrt{\E\reg_1(T)/c})\\
    &= \gamma T\left(1- \gamma\sqrt{\frac{\E\reg_1(T)}{c}}\right) - \frac{\gamma}{c}\E\reg_1(T).
\end{align*}
Setting $\gamma = \frac{1}{2}\sqrt{\frac{c}{\E\reg_1(T)}}\land \frac{1}{2}$, we have 
\begin{align*}
    \E\bar\reg_2(T) = \Omega\left(\min\left(T\epsilon, \frac{T\sqrt{c}}{\sqrt{\E\reg_1(T)}}\right)\right),
\end{align*}
which completes the proof.
\end{proof}


\section{Proofs from Section~\ref{sec:epsilon_greedy_regret}}
\label{app:proof-etc}
\begin{lemma}
\label{lem:saddle_point}
For any fixed $\mu\in\Mcal$, after $S$ iterations of lines 11-14 of Algorithm~\ref{alg:explore_then_commit} it holds with probability $1-\delta$ that
\begin{align*}
    \E_{\pi \sim \hat Q_{\mu},x,\bar a} \Delta(\pi,\bar a(x),x) \leq \min_{\pi'\in \Pi} \E_{x,\bar a} \Delta(\pi',\bar a(x),x) + \epsilon + O\left(\frac{1}{B} + \frac{1}{\sqrt{BS}} + \sqrt{\frac{\log(|\Pi|/\delta)}{T_0}}\right).
\end{align*}
Further, it holds that 
\begin{align*}
    \frac{1}{T_0}&\left\langle \hat Q_{\mu}, \mu\sum_{t=1}^{T_0} \hat r_t(\cdot,x_t) + (1-\mu)\sum_{t=2T_0+1}^{3T_0}(1-\Delta(\cdot,\bar a(x_t),x_t))\right\rangle\\
    &\qquad\geq\frac{1}{T_0}\left\langle Q, \mu\sum_{t=1}^{T_0} \hat r_t(\cdot,x_t) + (1-\mu)\sum_{t=2T_0+1}^{3T_0}(1-\Delta(\cdot,\bar a(x_t),x_t))\right\rangle - O\left(\sqrt{\frac{B}{S}}\right).
\end{align*}
\end{lemma}
\begin{proof}[Proof of Lemma~\ref{lem:saddle_point}]
Fix $\mu$, let $\bar \epsilon = \min_{\pi' \in \Pi} \sum_{t=2T_0+1}^{3T_0} \Delta(\pi'(x_t), \bar a(x_t), x_t) + \epsilon$.
Recall that $(\hat Q_\mu, \hat\lambda_{\mu})$ is the uniform mixture over $\{(Q_{s,\mu},\lambda_{s,\mu})\}_{s\in [S]}$. Best response and the MWU guarantee with step-size $\eta = \sqrt{\frac{1}{SB}}$ imply
\begin{align*}
    \hat \Lcal(Q,\hat\lambda_{\mu}) 
    &= \frac{1}{S}\sum_{s=1}^S \hat \Lcal(Q,\lambda_{s,\mu})\\ 
    &\leq  \frac{1}{S}\sum_{s=1}^S \hat \Lcal(Q_{s,\mu},\lambda_{s,\mu})\\
    &\leq \frac{1}{S}\sum_{s=1}^S \hat \Lcal(Q_{s,\mu}, \hat\lambda_\mu) + O\left(\sqrt{\frac{B}{S}}\right)\\
    &= \hat \Lcal(\hat Q_{\mu}, \hat\lambda_\mu) + O\left(\sqrt{\frac{B}{S}}\right)~,
\end{align*}
for any $Q\in\Delta(\pi)$. Similarly, in the other direction, we have
\begin{align*}
    \hat \Lcal(\hat Q_{\mu},\lambda) 
    &= \frac{1}{S}\sum_{s=1}^S \hat \Lcal(Q_{s,\mu},\lambda)\\ 
    &\geq \frac{1}{S}\sum_{s=1}^S \hat \Lcal(Q_{s,\mu},\lambda_{s,\mu}) - O\left(\sqrt{\frac{B}{S}}\right)\\
    &\geq \frac{1}{S}\sum_{s=1}^S \hat \Lcal(Q_{\mu},\lambda_{s,\mu}) - O\left(\sqrt{\frac{B}{S}}\right)\\ 
    &= \hat \Lcal(\hat Q_{\mu}, \hat\lambda_\mu) - O\left(\sqrt{\frac{B}{S}}\right)
\end{align*}
for any $\lambda \in [0,B]$.
Using the above approximate saddle point property together with Lemma 1 from \citet{agarwal2018reductions} we have
\begin{align*}
    \hat\lambda_{\mu}(\bar\epsilon - \langle \hat Q_{\mu}, \Delta(\cdot,\bar a(\cdot),\cdot \rangle) \leq B(\bar\epsilon - \langle \hat Q_{\mu}, \Delta(\cdot, \bar a(\cdot),\cdot \rangle)_{-} + O(\sqrt{B/S})~,
\end{align*}
where $(x)_{-} = \min(0,x)$.
For any feasible $Q$ combining the above inequality with $\hat \Lcal(\hat Q_{\mu}, \hat\lambda_\mu) \geq \hat L(Q,\hat\lambda_{\mu}) - O(\sqrt{B/S})$ implies
\begin{align*}
    \frac{1}{T_0}&\left\langle \hat Q_{\mu}, \mu\sum_{t=1}^{T_0} \hat r_t(\cdot,x_t) + (1-\mu)\sum_{t=T_0+1}^{2T_0}(1-\Delta(\cdot,\bar a(x_t),x_t))\right\rangle
    +B(\bar\epsilon - \langle \hat Q_{\mu}, \Delta(\cdot, \bar a(\cdot),\cdot \rangle)_{-} + O(\sqrt{B/S})\\
    &\geq\frac{1}{T_0}\left\langle Q, \mu\sum_{t=1}^{T_0} \hat r_t(\cdot,x_t) + (1-\mu)\sum_{t=T_0+1}^{2T_0}(1-\Delta(\cdot,\bar a(x_t),x_t))\right\rangle.
\end{align*}
We now use the above display to argue that
\begin{align*}
    &\frac{1}{T_0}\left\langle \hat Q_{\mu}, \mu\sum_{t=1}^{T_0} \hat r_t(\cdot,x_t) + (1-\mu)\sum_{t=T_0+1}^{2T_0}(1-\Delta(\cdot,\bar a(x_t),x_t))\right\rangle\\
    &\qquad\geq\frac{1}{T_0}\left\langle Q, \mu\sum_{t=1}^{T_0} \hat r_t(\cdot,x_t) + (1-\mu)\sum_{t=T_0+1}^{2T_0}(1-\Delta(\cdot,\bar a(x_t),x_t))\right\rangle - O\left(\sqrt{\frac{B}{S}}\right),
\end{align*}
and
\begin{align*}
    \frac{1}{T_0}\left\langle \hat Q_{\mu}, \sum_{t=2T_0+1}^{3T_0} \Delta(\cdot,\bar a(x_t),x_t)\right\rangle \leq \min_{\pi' \in \Pi} \sum_{t=2T_0+1}^{3T_0} \Delta(\pi'(x_t), \bar a(x_t), x_t) + \epsilon + O\left(\frac{1}{B} + \frac{1}{\sqrt{BS}}\right).
\end{align*}
The application of Azuma-Hoeffding's inequality, together with a union bound over $\Pi$ implies
\begin{align*}
    \E_{\pi \sim \hat Q_{\mu},x,\bar a} \Delta(\pi,\bar a(x),x) \leq \min_{\pi'\in \Pi} \E_{x,\bar a} \Delta(\pi',\bar a(x),x) + \epsilon - O\left(\frac{1}{B} + \frac{1}{\sqrt{BS}} + \sqrt{\frac{\log(|\Pi|/\delta)}{T_0}}\right),
\end{align*}
with probability $1-\delta$.
\end{proof}

Following \citet{zhang2019warm}, we select 
\begin{align*}
    \hat \mu= \argmax_{\mu\in \mathcal{M}} \frac{1}{T_0} \left\langle \hat Q_{\mu}, \sum_{t=T_0+1}^{2T_0} \hat r_t(\cdot,x_t)\right\rangle,
\end{align*}
and play according $\hat Q_{\hat\mu}$ for the rest of the game. Note that we need to sample a fresh batch of rewards as we do not have the martingale structure of Algorithm 1 from \citet{zhang2019warm}. 
We sample a fresh batch of $T_0$ rewards over which we carry out the union bound. Lemma~\ref{lem:saddle_point} already guarantees that $\hat Q_{\hat \mu}$ is going to be approximately feasible. It remains to show that $\hat Q_{\hat \mu}$ also attains a favorable reward. 

\begin{proof}[Proof of Theorem~\ref{thm:etc_main}]
Recall that
\begin{align*}
    V_{T_0}(\mu) = 2\sqrt{2 T_0(\mu^2 K + (1-\mu)^2)\log(4|\Pi|/\delta)} + (\mu K + (1-\mu))\log(4|\Pi|/\delta).
\end{align*}
For any $\pi \in \Pi$ let 
\begin{align*}
    \hat R_{\mu,T_0}(\pi) &= \mu\sum_{t=1}^{T_0} \hat r_t(\cdot,x_t) + (1-\mu)\sum_{t=2T_0+1}^{3T_0}(1-\Delta(\cdot,\bar a(x_t),x_t))\\
    R_{\mu,T_0}(\pi) &= T_0(\mu\E[r(\pi(x),x)] + (1-\mu)\E[(1-\Delta(\pi(x),\bar a(x),x))] ).
\end{align*}
Using Bernstein's inequality with the fact that we have done uniform exploration to construct $\hat r_t$ it holds with probability $1-\delta$ that, for all $\pi\in\Pi$:
\begin{align*}
    |\hat R_{\mu,T_0}(\pi) - R_{\mu,T_0}(\pi)| \leq V_{T_0}(\mu).
\end{align*}
Consequently, the same conclusion also holds for any $Q\in\Delta(\Pi)$. Conditioned on the above event, using the second part of Lemma~\ref{lem:saddle_point} we have that for any $Q \in\Delta(\Pi)$
\begin{align*}
    \langle Q, R_{\mu,T_0} \rangle - \langle \hat Q_{\mu}, R_{\mu,T_0} \rangle &\geq \langle Q, R_{\mu,T_0} \rangle - \langle Q, \hat R_{\mu,T_0} \rangle + \langle \hat Q_{\mu}, \hat R_{\mu,T_0} \rangle - \langle \hat Q_{\mu}, R_{\mu,T_0} \rangle - O\left(T_0\sqrt{\frac{B}{S}}\right)\\
    &\geq 2V_{T_0}(\mu) - O\left(T_0\sqrt{\frac{B}{S}}\right).
\end{align*}

To complete proceed further, we need the statement of Lemma 4~\citep{zhang2019warm} but adapted to the modified notion of $(\alpha,\dfrak)$-similarity. We restate the lemma below.
\begin{lemma}
\label{lem:lem4_zhang}
Assume that $D_1,D_2$ are $(\alpha, \dfrak)$-similar according to Definition~\ref{def:similarity}. Further suppose that
\begin{align*}
    \mu\E_{D_1}[r_1(\pi^*(x),x) - r_1(\pi(x),x)] + (1-\mu)\E_{D_2}[r_2(\pi^*(x),x) - r_2(\pi(x),x)] \leq R.
\end{align*}
Then it holds that
\begin{align*}
    \E_{D_1}[r_1(\pi^*(x),x) - r_1(\pi(x),x)] \leq \frac{R + (1-\mu)\dfrak}{\alpha(1-\mu) + \mu}.
\end{align*}
\end{lemma}
\begin{proof}
For ease of notation let $r_1^* = \E_{D_1}[r_1(\pi^*(x),x)], r_1 = \E_{D_1}[r_1(\pi(x),x)]$ and we use a similar notation for $r_2^*, r_2$.
The $(\alpha, \dfrak)$-similarity assumption implies that
\begin{align*}
    \frac{r^*_2 - r_2}{\alpha} + \frac{\dfrak}{\alpha} \geq r^*_1 - r_1.
\end{align*}
Next, plugging into the $R$-bound from the assumption of the lemma we have
\begin{align*}
    R &\geq \frac{\mu(r_2^* - r_2)}{\alpha} + (1-\mu)(r_2^* - r_2) + \frac{\mu}{\alpha}\dfrak\\
    &\iff\\
    \frac{R - \frac{\mu}{\alpha}\dfrak}{\frac{\mu}{\alpha} + (1-\mu)} &\geq r_2^* - r_2.
\end{align*}
Plugging back into the $(\alpha, \dfrak)$-similarity condition we have
\begin{align*}
    \alpha(r_1^* - r_1) &\leq \frac{R - \frac{\mu}{\alpha}\dfrak}{\frac{\mu}{\alpha} + (1-\mu)} + \dfrak\\
    &\iff\\
    r_1^* - r_1 &\leq \frac{R + (1-\mu)\dfrak}{\alpha(1-\mu) + \mu}.
\end{align*}
\end{proof}

Using this lemma, we have under $(\alpha,\mathfrak{d})$-similarity between $r$ and $\Delta$, that
\begin{align*}
    \left(\E\left[\langle Q, r \rangle\right] - \E\left[\langle \hat Q_{\mu}, r \rangle\right]\right)(\mu + \alpha(1-\mu)) \leq O\left(\frac{2V_{T_0}(\mu)}{T_0} + \sqrt{\frac{B}{S}} + (1-\mu)\mathfrak{d}\right),
\end{align*}
for any $Q\in\Delta(\Pi)$.
An application of Hoeffding's inequality with a union bound now implies that
\begin{align*}
    \max_{\mu\in \mathcal{M}}\E\left[\langle \hat Q_{\mu}, r\rangle\right] - \E\left[\langle \hat Q_{\hat \mu}, r\rangle\right] \leq   \sqrt{\frac{K\log(|\mathcal{M}|/(2\delta))}{T_0}},
\end{align*}
with probability at least $1-\delta/2$.
Combining with the previous display we have that for any $Q\in\Delta(\Pi)$ with probability $1-\delta$ it holds that 
\begin{align*}
    \E\left[\langle Q, r \rangle\right] - \E\left[\langle \hat Q_{\hat \mu}, r\rangle\right] \leq \sqrt{\frac{K\log(|\mathcal{M}|/(2\delta))}{T_0}} + O\left(\min_{\mu \in \mathcal{M}} \frac{\frac{2V_{T_0}(\mu)}{T_0} + \sqrt{\frac{B}{S}} + (1-\mu)\mathfrak{d}}{\mu + \alpha(1-\mu)}\right).
\end{align*}
We can easily convert the above high probability bound to a bound in expectation by noting that $\langle Q,r \rangle \leq 1, \forall Q\in \Delta(\Pi)$. Let the event that the above inequality holds be denoted by $\Ecal$. Setting $\delta = O(1/T_0)$ implies 
\begin{align*}
    \E[\langle Q - \hat Q_{\hat \mu}, r \rangle] &\leq \E[\langle Q - \hat Q_{\hat \mu}, r \rangle | \Ecal] + \frac{1}{T_0} \E[\langle Q - \hat Q_{\hat \mu}, r \rangle | \bar \Ecal]\\
    &\leq \sqrt{\frac{K\log(|\mathcal{M}|/(2\delta))}{T_0}} + O\left(\min_{\mu \in \mathcal{M}} \frac{\frac{2V_{T_0}(\mu)}{T_0} + \sqrt{\frac{B}{S}} + (1-\mu)\mathfrak{d}}{\mu + \alpha(1-\mu)}\right) + \frac{1}{T_0}.
\end{align*}
The bound on $\reg_r(\hat Q_{\hat\mu}, T)$ follows by using the above inequality for $t \geq 4T_0$ and bounding the regret in the first $4T_0$ by $4T_0$.
Finally, the bound on $\reg_c(\hat Q_{\hat\mu}, T)$ follows by using the first part of Lemma~\ref{lem:saddle_point} together with a similar argument to the above.
\end{proof}

\begin{lemma}
\label{lem:mcal_choice}
Assume that $\dfrak \in [0,1]$ and $\alpha \in [0, T]$. For the choice $\Mcal = \{1 - \frac{1}{2^n}, 1/K + \frac{1}{2^n}: n \leq \log(T)\}$ it holds that 
\begin{align*}
    \min_{\mu \in [0,1]} \frac{T^{2/3}\sqrt{(\mu^2 K + (1-\mu)^2)\log(|\Pi|T)} + T(1-\mu)\dfrak}{\mu + \alpha(1-\mu)}
    =O\left(\min_{\mu \in \Mcal} \frac{T^{2/3}\sqrt{(\mu^2 K + (1-\mu)^2)\log(|\Pi|T)} + T(1-\mu)\dfrak}{\mu + \alpha(1-\mu)}\right).
\end{align*}
\end{lemma}
\begin{proof}
Let $\mu^*$ be a solution to 
$$
\min_{\mu \in [0,1]} \frac{T^{2/3}\sqrt{(\mu^2 K + (1-\mu)^2)\log(|\Pi|T)} + T(1-\mu)\dfrak}{\mu + \alpha(1-\mu)}~.
$$ 
We show that there exists $\mu\in\Mcal$ such that 
    \begin{align*}
    \frac{T^{2/3}\sqrt{((\mu^*)^2 K + (1-\mu^*)^2)\log(|\Pi|T)} + T(1-\mu^*)\dfrak}{\mu^* + \alpha(1-\mu^*)}
    = O\left(\frac{T^{2/3}\sqrt{(\mu^2 K + (1-\mu)^2)\log(|\Pi|T)} + T(1-\mu)\dfrak}{\mu + \alpha(1-\mu)}\right).
    \end{align*}
Consider $\mu^* \geq \frac{1}{2}$ and write $\mu^* = 1- \frac{1}{2^\beta}$. We only consider large $K$ so that in this case $(\mu^*)^2 K \geq (1-(\mu^*)^2)$. If $\mu^* = 1$ then we can take $\mu = 1 - \frac{1}{T}$ and the claim is satisfied as $\dfrak \leq 1$ and $\alpha \leq T$. We can now consider $\mu^* \leq 1 - \frac{1}{T}$ and in particular we can take $\mu^* = 1 - \frac{1}{2^\beta}, \beta \in \mathbb{R}$. Let $\mu$ be the smallest $\mu \in \Mcal$ which exceeds $\mu^*$ and notice that $\mu \leq 1 - \frac{1}{2^{\beta+1}}$. We first compute 
\begin{align*}
    |\sqrt{K}\mu^* - \sqrt{K}\mu| \leq \sqrt{K}\left|\frac{1}{2^\beta} - \frac{1}{2^\beta+1}\right| \leq \frac{\sqrt{K}}{2^{\beta+1}} \leq \sqrt{K}\mu^*.
\end{align*}
The above already implies 
$$
\sqrt{(\mu^2 K + (1-\mu)^2)\log(|\Pi|T)} = O(\sqrt{((\mu^*)^2 K + (1-\mu^*)^2)\log(|\Pi|T)})~.
$$
Next, we consider
\begin{align*}
    |(1-\mu^*)T\dfrak - (1-\mu)T\dfrak| \leq \frac{T\dfrak}{2^{\beta+1}} \leq (1-\mu^*)T\dfrak,
\end{align*}
and so $T(1-\mu)\dfrak \leq 2(1-\mu^*)T\dfrak$. Overall we have shown that the numerators are within a constant factor of each other. Next, we consider the denominator. First, consider $\alpha \leq 1$, we have $\mu(1-\alpha) + \alpha \geq \mu^*(1-\alpha) + \alpha$ just by choosing $\mu \geq \mu^*$. Next, consider $\alpha > 1$:
\begin{align*}
    \mu(1-\alpha) + \alpha - \mu^*(1-\alpha) - \alpha = (\alpha - 1)(\mu^* - \mu).
\end{align*}
We first show $\alpha - \mu^*(\alpha - 1) \geq 2(\alpha - 1)(\mu - \mu^*)$ in the following way
\begin{align*}
    \alpha - \mu^*(\alpha - 1) &\geq 2(\alpha - 1)(\mu - \mu^*)\\
    &\iff\\
    \alpha - \mu(\alpha - 1) &\geq (\alpha-1)(\mu - \mu^*)\\
    &\iff 
    (\alpha - 1)(1-\mu) + 1 \geq (\alpha-1)(\mu - \mu^*)\\
    &\iff\\
    (\alpha - 1)(\mu - \mu^* - 1 + \mu) &\leq 1\\
    &\impliedby\\
    (\alpha - 1)\left(\frac{1}{2^{\beta}} - \frac{1}{2^{\beta+1}} - 1 + 1- \frac{1}{2^{\beta+1}}\right) &\leq 1,
\end{align*}
where the last inequality holds since $\frac{1}{2^{\beta}} - \frac{1}{2^{\beta+1}} - 1 + 1- \frac{1}{2^{\beta+1}} = 0$.
Thus we have
\begin{align*}
    \alpha - \mu(\alpha - 1) = \alpha - \mu^*(\alpha - 1) - (\mu - \mu^*)(\alpha-1) \geq \frac{1}{2}(\alpha - \mu^*(\alpha - 1)),
\end{align*}
which completes the proof that if $\mu^* \geq \frac{1}{2}$ we have
\begin{align*}
    \frac{T^{2/3}\sqrt{((\mu^*)^2 K + (1-\mu^*)^2)\log(|\Pi|T)} + T(1-\mu^*)\dfrak}{\mu^* + \alpha(1-\mu^*)}
    = O\left(\frac{T^{2/3}\sqrt{(\mu^2 K + (1-\mu)^2)\log(|\Pi|T)} + T(1-\mu)\dfrak}{\mu + \alpha(1-\mu)}\right).
\end{align*}
The case $\mu^* < \frac{1}{2}$ can be handled in a similar way, where we choose $\mu = 1 - \frac{1}{2^{\beta-1}}$.
\end{proof}

\section{Algorithm~\ref{alg:hedged_ftrl} and regret guarantees}
In this section we give more details on deriving Algorithm~\ref{alg:hedged_ftrl} and the regret guarantees from Section~\ref{sec:delta_estimators}.
\label{app:ftrl_proofs}
\subsection{Exp4 with constraint estimator and elimination}
\label{sec:exp4}

We now present an adaptation of the classical Exp4 algorithm~\citep{auer2002nonstochastic} to our problem. Since Exp4 only optimizes rewards without any constraints, we make two crucial modifications to it. First, we allow it to incorporate an arbitrary estimator $\bar \Delta$ for $\Delta(a, \atrue(x); x)$, and secondly, we incorporate a restriction of the policy class to policies which are approximately feasible under an appropriate constraint in terms of $\bar \Delta$. We now describe these two changes formally. 

\textbf{Approximate constraint oracle.} For the constraint, we assume for now that there exists an oracle which outputs a martingale sequence $(\bar \Delta_t)_{t\in[T]}$ such that $\bar\Delta_t$ is a good approximation to $\Delta$. Next, we clarify what is meant by good approximation.
\begin{assumption}
\label{assm:biased_est}
    There exists an oracle which at every time step $t\in [T]$ outputs $\bar \Delta_t(\cdot; x_t) : \Acal \to [0,1]$ s.t. $\bar \Delta_t(a; x_t) - \E_t[\bar \Delta_t(a; x_t)]$ forms a martingale difference sequence,\footnote{$\E_t$ denotes expectation conditioned on the observed history by the algorithm, up to and including all random quantities at round $t$ other than $r_t$.} for any action $a\in [K]$. Further, we assume that: $\E_t[\bar \Delta_t(a; x_t)^2] \leq v_t^2$, $\bar \Delta_t(a; x_t) \leq b$ and finally $|\E_t[\Delta(a, \bar a(x_t); x_t) - \bar \Delta_t(a; x_t)]| \leq \beta_t~\forall~a \in [K], t\in[T]$.
\end{assumption}
Assumption~\ref{assm:biased_est} allows us to use $\Delta_t$ as a proxy to $\Delta$ in two ways. First, we can use $\Delta_t$ as part of the reward feedback to the algorithm as we have done in Algorithm~\ref{alg:explore_then_commit}. Further, we can construct a sequence of nested policy sets which roughly limit the policy class to feasible policies, as we describe next.

\textbf{Nested policy sets.} Given an estimator $\bar \Delta_t$ satisfying Assumption~\ref{assm:biased_est}, it is natural to expect that if a policy has a small value of $\sum_{s=1}^t \bar \Delta_s(\pi(x_s); x_s)$, then it will also have a small value of $\E[\Delta(\pi(x), \atrue(x); x)]$, up to an $\epsilon_t = \tilde O(\frac{\sqrt{\sum_{s=1}^t v_s^2} + \sum_{s=1}^t \beta_s + b}{t})$ error coming from standard concentration arguments. Using this intuition, we can use a constraint estimator $\bar \Delta_t$ to construct a set of approximately feasible policies. For consistency of the approach, we need two crucial properties of the policy sets that we define next. 
\begin{definition}
Let $\pi^*$ be a solution to \eqref{eq:obj}. A nested sequence of policy sets $(\Pi_t)_{t\in[T]}$, with $\Pi_t \subseteq \Pi_{t-1}$ and  $\Pi_1 = \Pi$ is $((\epsilon_t)_{t\in[T]}, \delta)$ feasible if and only if with probability $1-\delta$, 
\begin{align*}
    \pi^* &\in \Pi_T\quad\mbox{and}\quad\forall\pi\in\Pi_t~:~\frac{\reg_c(\pi, t)}{t} \leq  \epsilon + \epsilon_t.
\end{align*}
\end{definition}
%
Under Assumption~\ref{assm:biased_est} we are able to construct the following $((\epsilon_t)_{t\in[T]},\delta)$-feasible nested policy sequence $(\Pi_t)_{t\in[T]}$.
Define
\begin{equation}
\label{eq:nested_seq_abstract}
    \begin{aligned}
        \Pi_1 &= \Pi\\
        \Pi_{t+1} &= \Bigg\{\pi \in \Pi_t : \sum_{s=1}^t \bar \Delta_s(\pi(x_s), x_s) \leq \min_{\pi \in \Pi}\sum_{s=1}^t \bar \Delta_s(\pi(x_s), x_s) + \epsilon\\
        &\qquad+ \sqrt{2\sum_{s=1}^t v_s^2\log(T|\Pi|/\delta)} + 2b\log(T|\Pi|/\delta) + \sum_{s=1}^t \beta_s \Bigg\}~.
    \end{aligned}
\end{equation}
The next result shows the properties of the sequence of policy sets defined in Equation~\ref{eq:nested_seq_abstract}.
Let 
\begin{align*}
    \bar \pi &= \argmin_{\pi \in \Pi}\E[\Delta(\pi(x),\bar a(x), x)]\\
    \bar \Pi &= \{\pi \in \Pi : \E[\Delta(\pi(x),\bar a(x), x)] \leq \E[\Delta(\bar \pi(x),\bar a(x), x)] + \epsilon\}~.
\end{align*}

\begin{lemma}
\label{lem:shrinking_pol_set2}
    For every round $t\in[T]$ it holds that $\bar \Pi \subseteq \Pi_t$ and further if $\pi \in \Pi_t$ then 
    \begin{align*}
        \E[\Delta(\pi(x),\bar a(x), x)] &\leq \E[\Delta(\bar \pi(x),\bar a(x), x)] + \epsilon\\
        &\qquad+ \frac{2}{t}\left(\sqrt{2\sum_{s=1}^t v_s^2\log(T|\Pi|/\delta)} + 2b\log(T|\Pi|/\delta) + \sum_{s=1}^t \beta_s\right)
    \end{align*}
with probability at least $1-\delta$. 
\end{lemma}
\begin{proof}
    Fix $\pi \in \Pi$. Freedman's inequality implies that 
    \begin{align*}
        \left|\sum_{s=1}^t \bar \Delta_s(\pi(x_s), x_s) - \E[\bar \Delta_s(\pi(x_s), x_s)| \Fcal_{t-1}]\right|
        \leq \sqrt{2\sum_{s=1}^t v_s^2\log(T/\delta)} + 2b\log(T/\delta),
    \end{align*}
with probability $1-\delta$ uniformly over all $t\in[T]$.
Combining with the bound on the bias 
$$
|\E[\hat \Delta_s(\pi(x_s), x_s) | \Fcal_{s-1}] - \E[\Delta(\pi(x_s), \bar a(x_s), x_s)| \Fcal_{s-1}]| \leq \beta_s
$$ 
we have
\begin{align*}
     \left|\frac{1}{t}\sum_{s=1}^t \bar \Delta_s(\pi(x_s), x_s) - \E[\Delta(\pi(x), \bar a(x), x)]\right|
     \leq \frac{1}{t}\left(\sqrt{2\sum_{s=1}^t v_s^2\log(T/\delta)} + 2b\log(T/\delta) + \sum_{s=1}^t \beta_s\right)~.
\end{align*}
A union bound over $\pi \in \Pi$ implies that
\begin{align*}
    \frac{1}{t}\sum_{s=1}^t \bar \Delta_s(\pi(x_s), x_s) &\geq \E[\Delta(\pi(x), \bar a(x), x)] - \frac{1}{t}\left(\sqrt{2\sum_{s=1}^t v_s^2\log(T|\Pi|/\delta)} + 2b\log(T|\Pi|/\delta) + \sum_{s=1}^t \beta_s\right),\\
    \min_{\pi \in \Pi_t} \frac{1}{t}\sum_{s=1}^t \bar \Delta_s(\pi(x_s), x_s) &\leq \frac{1}{t}\sum_{s=1}^t \bar \Delta_s(\bar \pi(x_s), x_s) \leq \E[\Delta(\bar \pi(x), \bar a(x), x)]\\
    &\qquad+ \frac{1}{t}\left(\sqrt{2\sum_{s=1}^t v_s^2\log(T|\Pi|/\delta)} + 2b\log(T|\Pi|/\delta) + \sum_{s=1}^t \beta_s\right),
\end{align*}
with probability $1-\delta$. Combining the two inequalities together with the definition of $\Pi_t$ implies
\begin{align*}
    \E[\Delta(\pi(x),\bar a(x), x)] \leq \E[\Delta(\bar \pi(x),\bar a(x), x)] + \epsilon + \frac{2}{t}\left(\sqrt{2\sum_{s=1}^t v_s^2\log(T|\Pi|/\delta)} + 2b\log(T|\Pi|/\delta) + \sum_{s=1}^t \beta_s\right)
\end{align*} 
with probability $1-\delta$ for $\pi \in \Pi_t$, which shows the second part of the lemma.

For the first part of the lemma let $\bar \pi_t$ be the minimizer of $\min_{\pi \in \Pi_t} \frac{1}{t}\sum_{s=1}^t \bar \Delta_s(\pi(x_s), x_s)$ and suppose that $\pi$ is feasible. We have
\begin{align*}
    &\frac{1}{t}\sum_{s=1}^t \bar \Delta_s(\bar \pi_t(x_s), x_s) \geq \E[\Delta(\bar \pi_t(x), \bar a(x), x)] - \frac{1}{t}\left(\sqrt{2\sum_{s=1}^t v_s^2\log(T|\Pi|/\delta)} + 2b\log(T|\Pi|/\delta) + \sum_{s=1}^t \beta_s\right)\\
    &\qquad\geq \E[\Delta(\bar \pi(x), \bar a(x), x)] - \frac{1}{t}\left(\sqrt{2\sum_{s=1}^t v_s^2\log(T|\Pi|/\delta)} + 2b\log(T|\Pi|/\delta) + \sum_{s=1}^t \beta_s\right)\\
    &\frac{1}{t}\sum_{s=1}^t \bar \Delta_s(\pi(x_s), x_s) \leq \E[\Delta(\pi(x), \bar a(x), x)] + \frac{1}{t}\left(\sqrt{2\sum_{s=1}^t v_s^2\log(T|\Pi|/\delta)} + 2b\log(T|\Pi|/\delta) + \sum_{s=1}^t \beta_s\right),
\end{align*}
where the second inequality follows from the fact that $\bar\pi$ minimizes the penalty $\Delta$.
Combining the two inequalities above with the feasibility of $\pi$ completes the proof of the lemma.
\end{proof}

The proof of Lemma~\ref{lem:shrinking_pol_set2} further guarantees that $\bar \Delta_t$ is a good estimator of $\Delta$ which implies that any $(\alpha,\dfrak)$-similarity between constraint and rewards will also hold between $(\bar \Delta_t)_{t\in[T]}$ and the rewards. More generally, we also assume that the distribution of $\bar \Delta_t(\cdot; x_t)$ is $(\alpha,\dfrak)$-similar to the reward according to Definition~\ref{def:similarity}, and relate this to the similarity of the original $\Delta$ distribution in the next section.
Using these ingredients, a natural way to modify Exp4 for our problem is to maintain a distribution over the policy set $\Pi$ via the following updates. The updates give the $\bar \Delta$ oracle the ability to play the revealing action $a_0$ on some rounds, in which case the algorithm does not get feedback on the rewards and does not update its distribution over policies. We capture these rounds by an indicator $Z_t$, which is 1 whenever $a_0$ is queried, and not controlled by the Exp4 updates for now. Moreover, we define our Exp4 update as:
\begin{align}
\ell_{t,a_t} 
&= 
1 - r(a_t,x_t)\notag\\
\hat \ell_{t,a}^\mu 
&= (1-Z_t)\left(\mu \frac{\chi(a=a_t)\ell_{t,a_{t}}}{P_{t,a_t}} + (1-\mu)\bar\Delta_t(a; x_t)\right),\notag\\
\tilde \ell_t 
&= \Pcal_t\hat \ell_t, \tilde L_t = \tilde L_{t-1} + \tilde\ell_t,\notag\\
Q_{t+1} 
&= \argmin_{Q\in\Delta(\Pi_{t+1})} \langle Q, \tilde L_t \rangle + \Psi_{t+1}(Q)~.\label{eq:ftrl_exp4}
\end{align}
Next, we unpack the update in Eq. (\ref{eq:ftrl_exp4}). First, we have chosen to work with losses, rather than rewards, as this setting is more suitable to the Exp4 algorithm. For the indicator $Z_t$, we note that it depends on $x_t$ and the random variables in all prior $t-1$ rounds, but is independent of $a_t$, conditioned on the past. Further, we let $\Pcal_{t}\in [0,1]^{\Pi}\times K$ be the matrix whose $i$-th row contains the distribution induced by policy $\pi_i$ over the $K$ actions, and let $p_t = Q_t \Pcal_t$ be the distribution over actions $[K]$. Finally, we also define $\Psi_t(Q) = -\frac{1}{\eta_t} \sum_{\pi\in\Pi_t} Q(\pi)\ln Q(\pi)$ to be the (scaled) negative entropy regularizer. we show in Appendix~\ref{app:ftrl_proofs} that the updates in ~(\ref{eq:ftrl_exp4}) enjoy the following regret guarantee
%
\begin{theorem}
\label{thm:nested_exp4_bound}
For any fixed $\mu$, $((\epsilon_t)_t,\delta)$-feasible nested sequence of policy sets $(\Pi_t)_t$, and a sequence $\{\bar \Delta_t(\cdot; x_t)\}_\tfrak$ s.t. $\E[\bar \Delta_t(\cdot; x_t)^2|\Fcal_{t-1}] \leq v^2, \forall t\in[T]$, assume that the distribution over $\bar\Delta_t(\cdot; x_t)$, is $(\alpha, \dfrak)$-similar to $D_b$, with respect to $(\Pi, \pi^*)$, where $\pi^*$ is a solution to \eqref{eq:obj}. 
The expected regret of the algorithm is bounded as
\begin{align*}
    \E[\reg_r(\Acal, T)]
    = O\Biggl(\frac{\frac{V_T(\mu, v)}{T}+ (1-\mu)\dfrak}{\mu + \alpha(1-\mu)} + 
    \E\Bigg[\frac{1}{T}\sum_{t\in[T]}Z_t\Bigg]\Biggl)~,
\end{align*}
%
Further, the expected constrained violation of $\Acal$ is no larger than $\epsilon + \frac{1}{T}\sum_{t=1}^T\epsilon_t$.
\end{theorem}

To show Theorem~\ref{thm:nested_exp4_bound} we first begin with a standard result for the update in Equation~\ref{eq:ftrl_exp4}, however, adapted to the nested sequence of policies $(\Pi_t)_{t\in[T]}$.
\begin{lemma}
\label{lem:ftrl_ineq}
Let $\eta_\tfrak = \frac{\eta_0}{\sqrt{\tfrak}}$ be the step-size. For any $\pi \in \Pi_T$ it holds that
\begin{align*}
    \sum_{\tfrak=1}^\Tcal \langle Q_\tfrak - e_{\tfrak,\pi}, \tilde \ell_\tfrak \rangle \leq \eta_0\sum_{\tfrak=1}^\Tcal \sum_{\pi \in \Pi_\tfrak}\frac{Q_{\tfrak,\pi}\tilde\ell_{\tfrak,\pi}^2}{\sqrt{\tfrak}} + \frac{3\sqrt{\Tcal}}{2\eta_0}\log(|\Pi|)~.
\end{align*}
\end{lemma}
\begin{proof}
For notational convenience as the feasible set changes through iterations, let us define $\Delb(\Pi_t) \in \R^{|\Pi|}$ to be the set of all distributions over $\Pi_t$, lifted up to a $|\Pi|$-dimensional space, by setting all the other coordinates to 0. In other words, 
$$
\Delb(\Pi_t) = \Bigl\{Q\in\R^{|\Pi|}~:~Q(\pi) \geq 0, Q(\pi) = 0~\text{for}~\pi\notin\Pi_t,~\sum_{\pi\in\Pi_t} Q(\pi) = 1\Bigl\}~.
$$
Note that inside the feasible set $\Delb(\Pi_t)$, we can also write $\Psi_t(Q) = -\frac{1}{\eta_t}H(Q)$, where $H(Q)$ is the Shannon entropy of the $|\Pi|$-dimensional distribution, since $0\log 0 = 0$. 
For the proof, we recall some standard facts of convex analysis used in bounding the regret of FTRL algorithms. For a vector $L\in\R^{|\Pi|}$, let us define 
\begin{equation}
    \Phi_t(L) = \sup_{Q\in\Delb(\Pi_t)} \inner{L,Q} - \Psi_t(Q),
    \label{eq:conjugate}
\end{equation}
to be the Fenchel conjugate of $\Psi_t + I(\Delb(\Pi_t))$, where $I(A)$ is the indicator of the set $A$, which is 0 inside the set and infinity otherwise. Since $\Psi_t$ is $1/\eta_t$-strongly convex in the $\ell_1$ norm, $\Phi_t$ is $\eta_t$-smooth in the $\ell_\infty$ norm (see e.g. \citep[][Theorem 1]{nesterov2005smooth}).  In particular, $\Phi_t$ is differentiable and $\nabla \Phi_t(L)$ is a solution to the constrained optimization in \eqref{eq:conjugate}, so that 

\begin{equation}
    \nabla \Phi_t(-\tilde L_{t-1}) = Q_{t},\quad\mbox{and}\quad \Phi_t(L+\ell) = \Phi_t(L) + \inner{\ell,\nabla \Phi_t(L)} + \frac{\eta_t}{2} \|\ell\|_\infty^2.
    \label{eq:smooth-dual}
\end{equation}

Let $Q\in\Delb(\Pi_T)$ be any distribution which is feasible at all rounds. Then we have
\begin{align*}
    -\inner{\tilde L_T, Q} \leq \sup_{Q'\in\Delb(\Pi_T)} \inner{-\tilde L_T,Q'} - \Psi_T(Q') + \Psi_T(Q) = \Phi_T(-\tilde L_T) + \Psi_T(Q).
\end{align*}

On the other hand, we would like to upper bound $\inner{\tilde \ell_t, Q_t}$ using the smoothness of $\Phi_t$. While an upper bound is immediate from the smoothness in $\ell_\infty$ norm above, we need a more careful control in local norms for the desired bound in the bandit setting. To this end, we define $\bar\Psi_t(Q) = -\frac{1}{\eta_t} H(Q)$ for $Q \in \Delta(\Pi_t)$ to be a function of the $|\Pi_t|$-dimensional distribution and $\bar\Psi_t^\star$ to be its convex conjugate, when we restrict the maximization to the simplex. For any vector $v\in\R^{|\Pi|}$, we also define $\cP_{\Pi_t} v$ to be its truncation to the coordinates in $\Pi_t$. Then we have
\begin{align*}
    \Phi_t(L) =& \sup_{Q\in\Delb(\Pi_t)} \inner{L,Q} - \Psi_t(Q) = \sup_{Q \in \Delta(\Pi_t)} \inner{\cP_{\Pi_t} L,Q} - \bar\Psi_t(Q) = \bar\Psi_t^\star(\cP_{\Pi_t} L),
\end{align*}
and also that $\nabla \bar\Psi_t^\star(\cP_{\Pi_t}(\tilde L_{t-1})) = Q_t$. Using this, we can obtain 

\begin{align*}
    \Phi_t(-\tilde L_t) - \Phi_t(-\tilde L_{t-1}) + \inner{\tilde \ell_t, Q_t}
    &\leq \bar\Psi_t^\star(-\cP_{\Pi_t} \tilde L_t)- \bar\Psi_t^\star(-\cP_{\Pi_t}(\tilde L_{t-1})) - \inner{\tilde \ell_t, Q_t}\\
    &\leq \eta_t \sum_{\pi\in\Pi_t} Q_t(\pi) \tilde\ell_t(\pi)^2\\
    &= \eta_t \sum_{\pi\in\Pi} Q_t(\pi) \tilde\ell_t(\pi)^2~,
\end{align*}
where the last inequality uses the contractivity of the projection operator and Theorem 2.22 of~\citet{shalev2012online}.
Adding the two inequalities, we obtain that 

\begin{align*}
    \sum_{t=1}^T \inner{\tilde \ell_t, Q_t - Q} 
    &\leq \sum_{t=1}^T \Phi_t(-\tilde L_{t-1}) - \Phi_t(-\tilde L_{t}) + \sum_{t=1}^T\eta_t\sum_{\pi\in\Pi} Q_t(\pi)\tilde\ell_t(\pi)^2 + \Phi_T(-\tilde L_T) + \Psi_T(Q)\\
    &= \Phi_1(\tilde L_0) + \sum_{t=1}^{T-1} (\Phi_{t+1}(-\tilde L_t) - \Phi_t(-\tilde L_t)) + \sum_{t=1}^T\eta_t\sum_{\pi}Q_t(\pi)\tilde\ell_t(\pi)^2 + \Psi_T(Q)~,
\end{align*}
where the last equality rearranges terms. Now we focus on the summand

\begin{align*}
    \Phi_{t+1}(-\tilde L_t) - \Phi_t(-\tilde L_t)
    &= \sup_{Q \in \Delb(\Pi_{t+1})} \inner{-\tilde L_t, Q} - \Psi_{t+1}(Q) - \sup_{Q \in \Delb(\Pi_{t})} \inner{-\tilde L_t, Q} - \Psi_{t}(Q)\\
    &\stackrel{(a)}{\leq} \sup_{Q \in \Delb(\Pi_{t+1})} \Psi_t(Q) - \Psi_{t+1}(Q)\\
    &\stackrel{(b)}{=} \sup_{Q \in \Delb(\Pi_{t+1})} \left(-\frac{1}{\eta_t} + \frac{1}{\eta_{t+1}}\right) H(Q) \leq \frac{1}{2\eta_0\sqrt{t}} \ln|\Pi|,
\end{align*}
where the inequality $(a)$ follows since $\sup_x f(x) + g(x) \leq \sup_x f(x) + \sup_x g(x)$ and using the fact that $\Pi_{t+1} \subseteq \Pi_t$. $(b)$ recalls that $\Psi_t(Q) = -\frac{1}{\eta_t}H(Q)$ on the set $\Delb(\Pi_{t'})$ for any $t' \geq t$.
Substituting this in our earlier bound, and noting that $\Phi_1(\tilde L_0) = \Phi_1(0) \leq 0$, $\Psi_T(Q) \leq \frac{\sqrt{T}}{\eta_0} \ln |\Pi|$ completes the proof.
\end{proof}

We use the above lemma to show the following.
\begin{corollary}
\label{cor:exp4_bound}
    Let $\E[\bar \Delta_t(a,x)^2] \leq v^2, \forall a\in[K]$. For any $\mu \in [0,1]$ playing according to the Exp-4 update in Equation~\ref{eq:ftrl_exp4} guarantees
    \begin{align*}
        \mu\E\left[\sum_{\tfrak=1}^\Tcal \sum_{\pi\in \Pi_\tfrak} Q_{\tfrak}(\pi)\ell(\pi(x),x)\right] &+ (1-\mu)\E\left[\sum_{\tfrak=1}^\Tcal\sum_{\pi\in\Pi_\tfrak} Q_t(\pi)\bar\Delta_\tfrak(\pi(x), x)\right]\\
        &\qquad- \mu\E[\Tcal\ell(\pi(x),x)] - (1-\mu)\E\left[\sum_{\tfrak=1}^\Tcal\bar\Delta_\tfrak(\pi(x), x)\right]\\
        &\leq \frac{\log(|\Pi|)}{\eta_0}\E[\sqrt{\Tcal}] + \eta_0 (\mu^2K + (1-\mu)^2 v^2)\E[\sqrt{\Tcal}]
        + \sum_{t=1}^T\E[(1-Z_t)],
    \end{align*}
    for any $\pi \in \Pi_\Tcal$, where $\ell(a,x) = 1- r(a,x)$.
\end{corollary}
\begin{proof}
    Let $\E_{\tfrak}$ denote the conditional expectation with respect to the sigma algebra $\Fcal_{t}$ induced by the random variables $\{a_{1:\tfrak}, x_{1:t}, Z_{1:t}, \bar a_{1:t}\}$. 
    We can apply Lemma~\ref{lem:ftrl_ineq} to get
    \begin{align*}
        \sum_{\tfrak=1}^\Tcal \langle Q_\tfrak - e_{\tfrak,\pi},\tilde\ell_\tfrak \rangle \leq \eta_0\sum_{\tfrak=1}^\Tcal\sum_{\pi \in \Pi_\tfrak} \frac{Q_{\tfrak,\pi}\tilde\ell_{\tfrak,\pi}^2}{\sqrt{\tfrak}} + \frac{\sqrt{\Tcal}}{\eta_0}\log(|\Pi|).
    \end{align*}
    Next, we consider $\E_{\tfrak}[\sum_{a \in \Acal_\tfrak} Q_{\tfrak, a} \tilde\ell_{\tfrak,a}^2]$. For an action $a$, let us define $Q(a|x) = \sum_{\pi\in\Pi~:~\pi(x) = a} Q(\pi)$. Then we have
    
    \begin{align*}
        \E_{\tfrak}\left[\sum_{\pi \in \Pi_\tfrak} Q_{\tfrak, \pi} \tilde\ell_{\tfrak,\pi}^2\right] &= \sum_{\pi\in \Pi_\tfrak}Q_{\tfrak,\pi}\E_{\tfrak}[\tilde\ell_{\tfrak,\pi}^2]\\
        &= \sum_{\pi\in \Pi_\tfrak}Q_{\tfrak,\pi} Q_t(\pi(x_t)|x_t) \E_t\left(\mu\frac{\ell_{t,\pi(x_t)}}{Q_t(\pi(x_t)|x_t)} + (1-\mu) \bar\Delta (\pi(x_t),x_t)\right)^2\\
        &\leq 2\sum_{\pi\in \Pi_\tfrak}Q_{\tfrak,\pi} Q_t(\pi(x_t)|x_t) \E_t\left(\mu\frac{\ell_{t,\pi(x_t)}}{Q_t(\pi(x_t)|x_t)}\right)^2 + 2\sum_{\pi\in \Pi_\tfrak}Q_{\tfrak,\pi} Q_t(\pi(x_t)|x_t)(1-\mu)^2 \E_t\bar\Delta (\pi(x_t),x_t)^2.
    \end{align*}
    The second term is bounded by $2v^2$, so we focus on the first term, which can be simplified further using a standard argument as 
    
    \begin{align*}
        \sum_{\pi\in \Pi_\tfrak}Q_{\tfrak,\pi} Q_t(\pi(x_t)|x_t) \E_t\left(\mu\frac{\ell_{t,\pi(x_t)}}{Q_t(\pi(x_t)|x_t)}\right)^2 =& \sum_{a}\sum_{\pi\in\Pi~:~\pi(x_t) = a} Q_{t,\pi}\left[Q_t(a|x_t) \E_t\left(\mu\frac{\ell_{t,a}}{Q_t(a|x_t)}\right)^2\right]\\
        =& \sum_a \ell_t(a)^2 \leq K.
    \end{align*}
    Thus, the RHS of the regret bound is bounded as $2\eta_0\sqrt{\Tcal}(\mu^2 K + (1-\mu)^2v^2) + \frac{\sqrt{\Tcal}\log(|\Pi|)}{\eta_0}$.
    For the LHS of the regret we note that 
    \begin{align*}
        \E_{\tfrak}[\langle Q_\tfrak, \tilde \ell_t \rangle] &= \sum_{\pi \in \Pi_\tfrak}\E_\tfrak\left[Q_\tfrak(\pi)Z_t\left(\mu\ell(\pi(x_\tfrak),x_\tfrak) + (1-\mu)\bar\Delta_\tfrak(\pi(x_\tfrak), x_\tfrak)\right)\right]
    \end{align*}
    Let $l_t(\pi(x_t), x_t) = \mu\ell(\pi(x_\tfrak),x_\tfrak) + (1-\mu)\bar\Delta_\tfrak(\pi(x_\tfrak), x_\tfrak$.
    Consider $\E[\langle Q_t - e_{t,\pi}, l_t - \tilde \ell_t \rangle]$,
    \begin{align*}
        \E[\langle Q_t - e_{t,\pi}, l_t - \tilde \ell_t \rangle] = \E[(1-Z_t)\langle Q_t - e_{t,\pi}, l_t \rangle] \leq \E[(1-Z_t)],
    \end{align*}
    where the inequality follows from the fact that $\langle Q_t- e_{t,\pi}, l_t \rangle \leq 1$.
    Combining the above two displays we have that the LHS of the regret is bounded as
    \begin{align*}
        \E\left[\langle Q_\tfrak - e_{\tfrak,\pi}, l_t \rangle\right] &= \E\left[\langle Q_\tfrak - e_{\tfrak,\pi}, \tilde\ell_t\rangle\right] + \E\left[\langle Q_\tfrak - e_{\tfrak,\pi}, l_t - \tilde\ell_t\rangle\right]\\
        &\leq \E\left[\langle Q_\tfrak - e_{\tfrak,\pi}, \tilde\ell_t\rangle\right] + \E[(1-Z_t)].
    \end{align*}
    Summing over the $\Tcal$ rounds of the game and taking expectation finishes the proof.
\end{proof}

We can now show Theorem~\ref{thm:nested_exp4_bound} which is the main result for a fixed $\mu$.
\begin{proof}[Proof of Theorem~\ref{thm:nested_exp4_bound}]
    We use Corollary~\ref{cor:exp4_bound} together with Lemma~\ref{lem:lem4_zhang} and the $(\alpha, \dfrak)$-similarity of $\bar\Delta$ with the rewards. The theorem then follows by directly plugging in Corollary~\ref{cor:exp4_bound} into Lemma~\ref{lem:lem4_zhang} and the fact that $\pi^* \in \Pi_\Tcal$. 
    The second part of the theorem follows directly from the $((\epsilon_t)_t,\delta)$-feasibility of the nested policy sets.
\end{proof}

\subsection{Model selecting the best $\mu$}
The Exp4 update from Equation~\ref{eq:ftrl_exp4} only works for a fixed $\mu$. 
To achieve a bound similar to the one in Section~\ref{sec:epsilon_greedy_regret} we further use model selection for the best $\mu$ through corralling Exp4 algorithms~\citep{agarwal2017corralling}, each corresponding to a single value of $\mu$. To that end consider running the Hedged FTRL corralling algorithm described in \citep{foster2020adapting, marinov2021pareto}. We now instantiate the algorithm with $M = O(\log(T))$ and each base algorithm is a version of Equation~\ref{eq:ftrl_exp4} with $\mu \in \{1 - 1/2^n, 1/K + 1/2^n : n \leq \log(T)\}$. 
These base algorithms are $(1/2, R_m)$-stable\footnote{For the definition of stability we refer the reader to \cite{agarwal2017corralling}.} with
\begin{align*}
R_m = \E\left[\sqrt{2T\log(|\Pi|)(\mu_m^2 K + (1-\mu_m)^2)}\right].
\end{align*}
In the context of our work stability takes the following form. We fix an algorithm $\Bcal_m$. Suppose that the rewards environment for $\Bcal_m$ has been changed from observing a reward $r(a_t, x_t)$ at time $t$ and constructing loss estimator $\hat \ell_t^\mu$ based on $1 - r(a_t, x_t)$ to observing a reward $r'(a_t, x_t)$ equaling $\frac{r(a_t, x_t)}{\rho_t}$ with probability $\rho_t$ and $0$ otherwise. That is $r'(a_t, x_t)$ is an unbiased estimator of $r(a_t, x_t)$ however, its second moment is scaled by $\rho_t$. We say that $\Bcal_m$ is $(1/2, R_m)$-stable if its regret bound under this new environment changes $R$ to $\sqrt{\rho_m}R, \rho_m = \argmax_{t\in[T]}\rho_t$ and keeps the remaining terms fixed, that is $\Bcal_m$ still enjoys an average regret bound of
\begin{equation}
\small
    \begin{aligned}
    \label{eq:exp4_stab}
        O\left(\frac{\sqrt{\rho_m\frac{\log(T|\Pi|)(\mu_m^2K + (1-\mu_m)^2v_m^2)}{T}} + (1-\mu_m)\dfrak)}{\mu_m + \alpha(1-\mu_m)}\right).
    \end{aligned}
\end{equation}
Using the stability the next theorem is a corollary from Theorem 2~\citep{marinov2021pareto}. 
\begin{theorem}
\label{thm:corr_thm}
Given a collection of $M$ base algorithms, $(\Bcal_m)_{m=1}^M$ which are $(1/2, \sqrt{C_mT\log(|\Pi|)})$-stable and any $C \geq 0$, then there exists a setting of the Hedged Tsallis-Inf algorithm's parameters (depending on $C$) (Algorithm 2 \cite{marinov2021pareto}) so that the regret of Hedged Tsallis-Inf is bounded as
\begin{align*}
    \forall m\in [M]\, :\, \E[R(T)]
    \leq\, 2\max\left\{C, \frac{C_m}{C}\right\}\E\left[\sqrt{MT\log(|\Pi|)}\right] + \E[\sqrt{2MT}]~.
\end{align*}
\end{theorem}
We note that in Theorem~\ref{thm:corr_thm} we have taken the regret of $m$-th base algorithm to be $R_m = \sqrt{C_m T\log(|\Pi|)}$.
\begin{proof}[Proof of Theorem~\ref{thm:corr_thm}]
The setting of parameters and the proof using Corollary~\ref{cor:exp4_bound} follows exactly the same steps as in \cite{marinov2021pareto} and so we omit it.
\end{proof}

The regret of Algorithm~\ref{alg:hedged_ftrl} is bounded as follows.
\begin{theorem}
\label{thm:hedged_bound}
    Under the assumptions of Theorem~\ref{thm:nested_exp4_bound}, with probability at least $1-\delta$,
    Algorithm~\ref{alg:hedged_ftrl} with $Base_m$ given by~\eqref{eq:ftrl_exp4} satisfies
    \begin{align*}
        \E[\reg_r(\Acal, T)] = O\Bigg(\min_{\mu \in [0,1]} \E\Bigg[\phi(\mu, v, T, \dfrak) + \sum_{t=1}^T Z_t\Bigg] \Bigg).
    \end{align*}
    Furthermore, we have $\E[\reg_c(\Acal, T)] \leq \epsilon + \frac{1}{T}\sum_{t=1}^T \epsilon_t$. 
\end{theorem}

Note that the theorem suggests that we can have an $O(\sqrt{T})$ regret on both the reward and constraint violation, so long as $\bar \Delta_t$ and $\Pi_t$ are such that $\sum_{t=1}^T \epsilon_t = O(\sqrt{T})$, $v=O(1)$ and $\sum_{t=1}^T Z_t = O(\sqrt{T})$. Clearly, such estimators are not possible without further assumptions, due to the $\Omega(T^{2/3})$ lower bound from Theorem~\ref{thm:lower_bound}, and we present examples of favorable structures which allow such improved upper bounds in the following section.  

To show Theorem~\ref{thm:hedged_bound} we set $C = 1$ and $C_m = \mu^2_m K + (1-\mu_m)^2$. 
\begin{proof}[Proof of Theorem~\ref{thm:hedged_bound}]
    The regret bound follows from the stability guarantee in Equation~\ref{eq:exp4_stab} together with the result stated in Theorem~\ref{thm:corr_thm}. The constraint violation bound follows directly from the fact that every algorithm shares the policy set $\Pi_t$ at round $t$ and by the $((\epsilon_t)_t,\delta)$-feasibility assumption every policy in $\Pi_t$ violates the constraint by at most $\epsilon + \epsilon_t$ with probability $1-\delta$ uniformly over all $t\in[T]$.
\end{proof}

\section{Proofs from Section~\ref{sec:delta_estimators}}
\label{app:delta_estimators}
\paragraph{Bias of $\hat\Delta_t$.}
We have the following
\begin{align*}
    \hat \Delta_t(\pi(x_t), x_t) &= \xi_t \Delta(\pi(x_t), a_t, x_t) + (1-\xi_t)\Delta(\pi(x_t), \bar a(x_t), x_t)\\
    &\leq \xi_t (\Delta(\pi(x_t), \bar a(x_t), x_t) + \Delta(a_t, \bar a(x_t), x_t)) + (1-\xi_t)\Delta(\pi(x_t), \bar a(x_t), x_t)\\
    &\leq \nu + \Delta(\pi(x_t), \bar a(x_t), x_t).
\end{align*}
Similarly we have $\hat \Delta_t(\pi(x_t), x_t) \geq \Delta(\pi(x_t),\bar a(x_t),x_t) - \nu$, thus $\hat \Delta_t$ can be used to construct a $\nu$-biased estimator of $\Delta$. 
\paragraph{Properties of $\Pi_t$.}
We make two observations about $\Pi_t$, first it always contains the set of all feasible policies with probability $1-\delta$, and second any policy belonging to $\Pi_t$ violates the constraint by at most $2\left(\nu + \sqrt{\frac{\log(T|\Pi|/\delta)}{t}}\right)$. Both of this observations follow from the fact that $\{ \hat \Delta_t(\pi(x_t), x_t) - \E[\hat \Delta_t(\pi(x_t), x_t)]\}_t$ is a martingale difference sequence for every $\pi \in \Pi$.
Let $\bar \pi = \argmin_{\pi \in \Pi}\E[\Delta(\pi(x),\bar a(x), x)]$ and let $\bar \Pi = \{\pi \in \Pi : \E[\Delta(\pi(x),\bar a(x), x)] \leq \E[\Delta(\bar \pi(x),\bar a(x), x)] + \epsilon\}$.
\begin{proof}[Proof of Theorem~\ref{thm:biased_constraint}]
    We only need to argue two statements. First if $\Delta(\cdot, \bar a(x), x)$ is $(\alpha,\dfrak)$-similar to the reward distribution then $\bar\Delta_t$ is $(\alpha,\dfrak + \nu)$ similar and second, the sets $(\Pi_t)_{t\in[T]}$ are $((\epsilon_t)_t,\delta)$-feasible with $\epsilon_t \leq 4\nu + 8\sqrt{\frac{\log(T|\Pi|/\delta)}{t}}$. The first statement holds immediately from Definition~\ref{def:similarity} together with Assumption~\ref{assm:accepting_constr}. The second statement follows directly from Lemma~\ref{lem:shrinking_pol_set2}.
\end{proof}
\paragraph{Doubly robust estimator.}
\begin{lemma}
\label{lem:doubly_robust}
The doubly robust estimator 
\begin{align*}
    \bar \Delta_t(a, x_t) =  \hat \Delta_t(a, x_t) + Z_t\frac{(\Delta(a, \bar a(x_t), x_t) - \hat \Delta_t(a, x_t) )}{\gamma_t},
\end{align*}
is unbiased, that is $\E[\bar \Delta_t(a, x_t)] = \E[\Delta(a, \bar a(x_t), x_t)]$. Further we have 
$\E[\bar \Delta_t(a, x_t)^2] \leq 2 + 2\nu^2\E\left[\frac{Z_t}{\gamma_t^2}\right] = 2 + 2\frac{\nu^2}{\gamma_t}$ and $|\bar\Delta_t(a, x_t)| \leq 1, \forall a\in[K]$.
\end{lemma}
\begin{proof}
We note that 
\begin{align*}
    \E[\bar \Delta_t(a, x_t) | x_t, a] =& \hat \Delta_t(a, x_t) + \frac{(\Delta(a, \bar a(x_t), x_t) - \hat \Delta_t(a, x_t) )}{\gamma_t}\E_t[Z_t],
\end{align*}
since both $\hat\Delta$ and $\bar a(x_t)$ do not depend on the randomness in $Z_t$. Since $\E_t[Z_t] = \gamma_t$, this shows that $\bar \Delta_t$ is an unbiased estimator of $\Delta(a, \bar a(x_t), x_t)$. 

Next, we compute the variance. We can use the bias bound for $\hat \Delta$ to write
\begin{align*}
    \E_t[\bar \Delta_t(\pi(x_t), x_t)^2] \leq 2 + 2\nu^2\E_t\left[\frac{Z_t}{\gamma_t^2}\right] = 2 + 2\frac{\nu^2}{\gamma_t}.
\end{align*}
Finally, $|\bar\Delta_t(a,x_t)| \leq 1 + \frac{\nu}{\gamma_t}$ as $(\Delta(a, \bar a(x_t), x_t) - \hat \Delta_t(a, x_t) ) \leq \nu$. 

Second the variance is also bounded by $O(\frac{\nu^2}{\gamma_t})$, thus the conditions of Lemma~\ref{lem:shrinking_pol_set2} are met and we have that $(\Pi_t)_{t\in[T]}$ is $((\epsilon_t)_{t\in[T]},\delta)$-feasible with $\epsilon_t = O(U_t(\delta, \nu))$.
\end{proof}
\begin{proof}[Proof of Theorem~\ref{thm:biased_constraint2}]
The $(\alpha, \dfrak)$-similarity is immediate by the unbiasedness of the estimator guaranteed by Lemma~\ref{lem:doubly_robust}. Further, Lemma~\ref{lem:doubly_robust} implies that the conditions of Lemma~\ref{lem:shrinking_pol_set2} are met and we have that $(\Pi_t)_{t\in[T]}$ is $((\epsilon_t)_{t\in[T]},\delta)$-feasible with $\epsilon_t = O(U_t(\delta, \nu))$. Finally, the reward regret bound follows from the variance bound in Theorem~\ref{thm:hedged_bound}.
\end{proof}

\section{Proofs from Section~\ref{sec:active}}
\label{app:active_learning}
For the proof of Theorem~\ref{thm:active_learning} we recall the following definitions 
\begin{align*}
    \hat \pi_n &= \argmin_{\pi \in \Pi_n} \frac{1}{n} \sum_{i=1}^n Z_i\Delta(\pi(x_i), \bar a(x_i), x_i))\\
    r_n &= 4\sqrt{2\frac{\log(|\Pi|/\delta)}{n}} \\
    \Pi_{n+1} &= \left\{\pi \in \Pi_n : \frac{1}{n} \sum_{i=1}^n Z_i\Delta(\pi(x_i), \bar a(x_i), x_i) \leq \frac{1}{n} \sum_{i=1}^n Z_i\Delta(\hat\pi_n(x_i), \bar a(x_i), x_i) + 2\epsilon + 3r_{n+1}\right\}\\
    Z_{n+1} &= \chi\left(\exists \pi, \pi' \in \Pi_{n+1} : \Delta(\pi(x_{n+1}), \pi'(x_{n+1}), x_{n+1}) \geq \frac{2\epsilon + r_{n+1}}{2}\right).
\end{align*}
Further, recall that $\bar \pi = \argmin_{\pi\in\Pi}\E[\Delta(\pi(x),\bar a(x), x)]$
\begin{lemma}
\label{lem:version_space_incl}
It holds that
\begin{align*}
    \left\{\pi \in \Pi : \E[\Delta(\pi(x),\bar a(x), x)] \leq \E[\Delta(\bar\pi(x),\bar a(x), x)] + \epsilon\right\} \subseteq \Pi_{t}, \forall t \in [T]
\end{align*}
with probability $1-\delta$.
\end{lemma}
\begin{proof}
By definition of $Z_i$ and the fact that $\Pi_{t+1} \subseteq \Pi_t, \forall t \leq n$ we have that for any $\pi, \pi' \in \Pi_t$ and all $i\leq t$
\begin{align*}
    |(1-Z_i)\Delta(\pi(x_i),\bar a(x_i), x_i) - (1-Z_i)\Delta(\pi'(x_i), \bar a(x_i), x_i)| \leq (1-Z_i)\Delta(\pi(x_i),\pi'(x_i),x_i) \leq \frac{2\epsilon + r_i}{2}.
\end{align*}

First, by induction on $\bar\pi$ we show that 
\begin{align*}
    \sum_{i=1}^t Z_i(\Delta(\bar\pi(x_i), \bar a(x_i), x_i) - \Delta(\hat\pi_t(x_i), \bar a(x_i), x_i))
    \leq  t\epsilon + 3r_t.
\end{align*}
Proceed by induction on $\bar\pi$, and assume that $\bar\pi \in \Pi_i, \forall i\leq t$. We have the following
\begin{align*}
    &\sum_{i=1}^t Z_i(\Delta(\bar\pi(x_i), \bar a(x_i), x_i) - \Delta(\hat\pi_t(x_i), \bar a(x_i), x_i))\\ 
    &\qquad= \sum_{i=1}^t Z_i(\Delta(\bar\pi(x_i), \bar a(x_i), x_i) - \Delta(\hat\pi_t(x_i), \bar a(x_i), x_i))
    + \sum_{i=1}^t (1-Z_i)(\Delta(\bar\pi(x_i), \bar a(x_i), x_i)- \Delta(\hat\pi_t(x_i), \bar a(x_i), x_i))\\
    &\qquad\qquad - \sum_{i=1}^t (1-Z_i)(\Delta(\bar\pi(x_i), \bar a(x_i), x_i) - \Delta(\hat\pi_t(x_i), \bar a(x_i), x_i))\\
    &\qquad\leq \sum_{i=1}^t (\Delta(\bar\pi(x_i), \bar a(x_i), x_i) - \Delta(\hat\pi_t(x_i), \bar a(x_i), x_i)) + t\epsilon + \sum_{i=1}^t \frac{r_i}{2}\\
    &\qquad\leq t\E[\Delta(\bar\pi(x), \bar a(x), x) - \Delta(\hat\pi_t(x), \bar a(x), x)] + t\epsilon + \sum_{i=1}^t \frac{r_i}{2} + 2\sqrt{2t\log(1/\delta)}\\
    &\qquad\leq t\epsilon + \sum_{i=1}^t \frac{r_i}{2} + 2\sqrt{2t\log(n/\delta)},
\end{align*}
where in the second to last inequality we used Azuma-Hoeffding and a union bound over $\Pi$, and in the last inequality we used the definition of $\bar\pi$. Setting $r_i = 4\sqrt{\frac{2\log(n|\Pi|/\delta)}{i}}$ completes the induction. 
Next, in the same way as in the induction step we can show that for any fixed $\pi \in \Pi$ 
it holds that
\begin{align*}
    \sum_{i=1}^t & Z_i(\Delta(\pi(x_i), \bar a(x_i), x_i) - \Delta(\hat\pi_t(x_i), \bar a(x_i), x_i))\\
    &\leq  t\E[\Delta(\pi(x), \bar a(x), x) - \Delta(\hat\pi_t(x), \bar a(x), x)] + t\epsilon + \sum_{i=1}^t \frac{r_i}{2} + 2\sqrt{2t\log(1/\delta)}\\
    &= t\E[\Delta(\bar\pi(x), \bar a(x), x) - \Delta(\hat\pi_t(x), \bar a(x), x)] + t\E[\Delta(\pi(x), \bar a(x), x) - \Delta(\bar\pi(x), \bar a(x), x)]\\
    &\qquad+ t\epsilon + \sum_{i=1}^t \frac{r_i}{2} + 2\sqrt{2t\log(1/\delta)}~.
\end{align*}
Using the fact that $\E[\Delta(\pi(x),\bar a(x), x)] \leq \E[\Delta(\bar\pi(x),\bar a(x), x)] + \epsilon$ together with the claim for $\bar\pi$ and the choice of $r_i$ the proof is complete.
\end{proof}

\begin{lemma}
\label{lem:version_space_ineq}
If $\bar\pi \in \Pi_n$ then $\E[\Delta(\pi(x), \bar a(x), x)] \leq \E[\Delta(\bar\pi(x), \bar a(x), x)] + 3\epsilon + 10r_n, \forall \pi \in \Pi_n$.
\end{lemma}
\begin{proof}
First we note that for any fixed $\pi \in \Pi$ we have that $\{Z_i\Delta(\pi(x_i), \bar a(x_i), x_i) - \E_i[Z_i\Delta(\pi(x_i), \bar a(x_i), x_i)]\}_{i}$ is a martingale difference sequence with respect to the filtration induced by $\{Z_{j}\}_{j=1}^{i-1}$. Let 
$$
Y_i = Z_i\Delta(\pi(x_i), \bar a(x_i), x_i) - \E_i[Z_i\Delta(\pi(x_i), \bar a(x_i), x_i)]~.
$$
Note that $Y_i \in [-1,1]$ and that $Y_i^2 \leq 1$ and so Freedman's inequality implies
\begin{align*}
    \P\left(\sum_{i=1}^t Y_i > \sqrt{2t\log(1/\delta)} + 2\log(1/\delta)\right) \leq \delta.
\end{align*}
Fix $\pi \in \Pi_n$. We have
\begin{align*}
    n\E[\Delta(\pi(x), \bar a(x), x) - \Delta(\bar\pi(x), \bar a(x), x)] &= \sum_{i=1}^n \P(Z_i = 0)\E[\Delta(\pi(x_i), \bar a(x_i), x_i) - \Delta(\bar\pi(x_i), \bar a(x_i), x_i) | Z_i = 0]\\
    &\qquad+ \sum_{i=1}^n \P(Z_i = 1)\E[Z_i(\Delta(\pi(x_i), \bar a(x_i), x_i) - \Delta(\bar\pi(x_i), \bar a(x_i), x_i))| Z_i = 1]\\
    &\leq \sum_{i=1}^n \E[\Delta(\pi(x_i), \bar\pi(x_i), x_i)| Z_i = 0]\\
    &\qquad+ \sum_{i=1}^n \E[Z_i(\Delta(\pi(x_i), \bar a(x_i), x_i) - \Delta(\bar\pi(x_i), \bar a(x_i), x_i))]\\
    &\leq n\epsilon + \sum_{i=1}^n \frac{r_i}{2} + \sum_{i=1}^n \E[Z_i(\Delta(\pi(x_i), \bar a(x_i), x_i) - \Delta(\bar\pi(x_i), \bar a(x_i), x_i))]\\
    &\leq n\epsilon + \sum_{i=1}^n \frac{r_i}{2} + \sum_{i=1}^n Z_i(\Delta(\pi(x_i), \bar a(x_i), x_i) - \Delta(\bar\pi(x_i), \bar a(x_i), x_i))\\
    &\qquad+ 2\sqrt{2n\log(1/\delta)} + 4\log(1/\delta),
\end{align*}
where in the last inequality we used Freedman's inequality. Finally, using the definition of $\Pi_n$, together with the fact that both $\pi,\bar\pi \in \Pi$ we have
\begin{align*}
    \E[\Delta(\pi(x), \bar a(x), x) - \Delta(\bar\pi(x), \bar a(x), x)] \leq 3\epsilon + 10r_{n+1}.
\end{align*}
\end{proof}
Let
\begin{align*}
    \Pi(r) = \left\{\pi : \E[\Delta(\pi(x), \bar a(x), x)] \leq \E[\Delta(\bar\pi(x), \bar a(x), x)] + 3\epsilon + r\right\}.
\end{align*}
Lemma~\ref{lem:version_space_ineq} implies that $\Pi_n \subseteq \Pi(10r_n)$.
Now we define a low noise condition which weakens Assumption~\ref{assm:constraint-massart}.

\begin{assumption} For all $\pi \in \Pi$, we have that one of the following conditions holds:
\begin{equation*}
        \text{ either } \E[\Delta(\pi(x), \bar a(x); x)]
        \geq \E[\Delta(\bar \pi(x), \bar a(x); x)] + 3\epsilon + \tau\quad
        \text{ or }\quad \Delta(\pi(x), \bar\pi(x); x) \leq \frac{2\epsilon + \tau}{4}, \forall x.
\end{equation*}
\label{assm:constraint-massart-weak}
\end{assumption}

Clearly when we have a pointwise margin, like in Assumption~\ref{assm:constraint-massart}, the above assumption also holds as we are never in the first case. We now bound the query complexity under this weaker assumption as follows.
\begin{align*}
    \sum_{i=1}^n \E[Z_i] &= \sum_{i=1}^n \P\left(\exists \pi, \pi' \in \Pi_i: \Delta(\pi(x_i),\pi'(x_i), x_i) \geq \frac{2\epsilon + r_i}{2}\right)\\
    &\leq \sum_{i=1}^n \P\left(\exists \pi, \pi' \in \Pi(10r_i): \Delta(\pi(x_i),\pi'(x_i), x_i) \geq \frac{2\epsilon + r_i}{2}\right).
\end{align*}
Under Assumption~\ref{assm:constraint-massart-weak}, we note that for any $i\geq \frac{80\log(|\Pi|n/\delta)}{\tau^2}$ with probability $1-\delta$ it holds that $\Pi(10r_i)$ contains only policies $\pi,\pi'$ such that $\Delta(\pi(x),\bar\pi(x),x) \leq \frac{2\epsilon + \tau}{4}, \Delta(\pi'(x),\bar\pi(x),x) \leq \frac{2\epsilon + \tau}{4}$ which implies $\Delta(\pi(x), \pi'(x), x) \leq \frac{2\epsilon + \tau}{2}$ and so $\E[Z_i] = 0$. Arguing for the query complexity as before gives the following lemma.
\begin{lemma}
\label{lem:query_complexity_massart}
Under Assumption~\ref{assm:constraint-massart-weak}, it holds that the query complexity of the active learner is at most $\frac{80\log(n|\Pi|/\delta)}{\tau^2}$ with probability $1-\delta$.
\end{lemma}
We note that Lemma~\ref{lem:version_space_ineq} implies that any $\pi \in \Pi_t$ violates the constraint by at most $3\epsilon + O(\sqrt{\log(T|\Pi|/\delta)/t})$ with probability $1-\delta$. 
Note that it is impossible to establish a meaningful $\Delta_t(\cdot, x_t)$ with a controlled bias against $\Delta(\cdot, \bar a(x_t), x_t)$, however, we can instead use a potential alignment of the losses with $\Delta(\cdot, \bar\pi(x), x)$. We can now complete the proof of Theorem~\ref{thm:active_learning}
\begin{proof}[Proof of Theorem~\ref{thm:active_learning}]
 Lemma~\ref{lem:query_complexity_massart} implies that the regret accumulated due to the active learner is at most $\frac{20\log(n|\Pi|/\delta)}{\tau^2}$ with probability $1-\delta$. This implies that the regret in expectation is at most $O\left(\frac{\log(n|\Pi|)}{\tau^2}\right)$. Further, Algorithm~\ref{alg:hedged_ftrl} sets $\bar \Delta_t(\pi(x), x) = \Delta(\pi(x), \hat\pi_t(x), x)$. Thus, on every round on which $Z_t = 0$, the active learning rule implies that 
\begin{align*}
    |\Delta(\pi(x), \hat\pi_t(x), x) - \Delta(\pi(x), \bar\pi(x), x)| \leq \Delta(\hat\pi_t(x), \bar\pi(x), x) \leq \epsilon + O\left(\sqrt{\frac{\log(\log(T|\Pi|/\delta))}{t}}\right).
\end{align*}
This implies that the distribution of $r(\cdot, x_t)$ is $\left(\alpha, \dfrak + \epsilon + O\left(\sqrt{\frac{\log(\log(T|\Pi|/\delta))}{t}}\right)\right)$-similar to $\bar\Delta_t$. Further, Lemma~\ref{lem:version_space_incl} implies that $\pi^* \in \Pi_T$ with probability $1-\delta$. Corollary~\ref{cor:exp4_bound} now finishes the proof.
\end{proof}

\end{document}